\documentclass{article}
\newcommand{\Lcal}{\mathcal{L}}

\usepackage[main,final]{neurips_2026} 

\usepackage[utf8]{inputenc}
\usepackage[T1]{fontenc}
\usepackage{courier}
\usepackage{amsmath,amssymb,amsthm}
\usepackage{thmtools}
\usepackage{mathtools}
\usepackage{hyperref}
\usepackage{url}
\bibpunct{[}{]}{,}{a}{,}{,}
\hypersetup{pdftitle={Factorized Spectral Representations for Reinforcement Learning},pdfauthor={Junyi Wu and Dan Li}}
\usepackage{booktabs}
\usepackage{nicefrac}
\usepackage{wrapfig}
\usepackage{microtype}
\usepackage{xcolor}
\usepackage{graphicx}
\usepackage{subcaption}
\usepackage{algorithm}
\usepackage{algorithmic}
\usepackage{multirow}
\usepackage{tabularx}
\usepackage{enumitem}

\theoremstyle{plain}
\newtheorem{theorem}{Theorem}[section]
\newtheorem*{theorem*}{Theorem}
\newtheorem{lemma}[theorem]{Lemma}

\newtheorem*{corollary*}{Corollary}
\newtheorem{assumption}{Assumption}
\newtheorem{proposition}[theorem]{Proposition}
\newtheorem*{proposition*}{Proposition}

\theoremstyle{definition}
\newtheorem{definition}{Definition}

\theoremstyle{remark}

\newcommand{\R}{\mathbb{R}}
\newcommand{\E}{\mathbb{E}}
\newcommand{\Scal}{\mathcal{S}}
\newcommand{\Acal}{\mathcal{A}}
\newcommand{\Fcal}{\mathcal{F}}
\newcommand{\Gcal}{\mathcal{G}}
\newcommand{\Hcal}{\mathcal{H}}

\newcommand{\Vcal}{\mathcal{V}}
\newcommand{\Ncal}{\mathcal{N}}
\newcommand{\Rcal}{\mathcal{R}}
\newcommand{\phis}{\phi_s}
\newcommand{\phia}{\phi_a}
\newcommand{\sg}{\operatorname{sg}}

\DeclareMathOperator{\vect}{vec}

\DeclareMathOperator{\diam}{diam}

\title{Factorized Spectral Representations for Reinforcement Learning}
\author{
Junyi Wu \\
University of Washington\\
Seattle, WA, USA\\
\texttt{junyiwu@uw.edu}
\And
Dan Li\\
University of Washington\\
Seattle, WA, USA\\
\texttt{dli27@uw.edu}
}
\begin{document}
\maketitle

\begin{abstract}
  Learning a compact model of the world from interaction data is central to sample-efficient deep reinforcement learning. Spectral representation methods have become the leading paradigm for representation learning in continuous control by taking a matrix view of the transition kernel, with state-action pairs on one side and next states on the other, and learning a low-rank factorization through self-supervised contrastive objectives. We take this view one step further. The transition kernel is naturally a three-mode tensor over states, actions, and next states, and a CANDECOMP/PARAFAC (CP) decomposition gives one feature map per mode. We propose FaStR, which fits this decomposition with a noise contrastive objective, producing separate state, action, and next-state encoders that together form a single spectral representation. The factored form yields a smaller hypothesis class, and under CP realizability the sufficient-sample upper bound for representation learning improves by a factor that scales with the smaller of the state and action dimensions. Empirically, FaStR delivers its largest gains on high-dimensional locomotion tasks whose dynamics align with the factored structure, and the learned state encoder transfers intact across actuator shift while only the action encoder is retrained.
\end{abstract}

\section{Introduction}
\label{sec:intro}

Sample efficiency is a central concern in continuous-control RL. As state and action dimensions grow, the number of environment interactions required by model-free methods rises sharply, limiting real-world deployment~\citep{fujimoto2018addressing,haarnoja2018soft}. One productive route is the framework of low-rank MDPs~\citep{jin2020provably,yang2020reinforcement,agarwal2020flambe}, which assumes the transition kernel admits a finite-rank factorization. The representation-learning problem then becomes one of finding a feature space in which the dynamics are linear. Spectral representation methods bring this low-rank view closer to practical deep RL~\citep{gao2025spectral}. They use transition samples from online interaction to learn transition-aware features, often without reward labels, and feed these features to policy and value learning. When such features are estimated more accurately from the same interaction budget, the critic receives a representation closer to the Bellman-linear structure earlier in training. The contrastive objective used to extract these features has been refined steadily across this line of work, and the resulting agents are now competitive on standard benchmarks.

RL representations preserve different kinds of information. Successor and forward--backward representations summarize long-term visitation and behavior~\citep{dayan1993improving,barreto2017successor,touati2021learning}. Laplacian and proto-value methods describe the geometry of the transition graph~\citep{mahadevan2005proto,wu2019laplacian}, while bisimulation, contrastive, and predictive methods learn state features from behavioral similarity or future information~\citep{ferns2004metrics,zhang2021learning,laskin2020curl,stooke2021decoupling,schwarzer2021data}. Spectral RL instead represents the one-step transition itself, which makes how the state and the action enter the representation a design choice.
In many control systems, states and actions play different roles. The state describes the current configuration of the system, while the action describes the input driving the system. Existing spectral RL methods~\citep{ren2022spectral,zhang2022making,gao2025spectral,shribak2024diffusion,tang2024spectral} implement low-rank structure by treating the state-action pair as a joint argument to the transition kernel. Separately, model-based methods learn explicit dynamics models for planning or imagination~\citep{hafner2023mastering,hansen2024tdmpc2}, while factored value-function methods place low-rank structure on the critic or the policy instead of the transition~\citep{mahajan2021tesseract,lobel2023qfunctionals,rozada2024tensor,rozada2025tensor}. This work identifies a gap between these two directions: the lack of a factored low-rank representation for the transition dynamics themselves. The motivation stems from the observation that distinct actions often drive identical state-dependent motion patterns. 
\begin{wrapfigure}{r}[0pt]{0.5\textwidth}
  \centering
  \begin{subfigure}{0.5\textwidth}
    \centering
    \includegraphics[width=\textwidth]{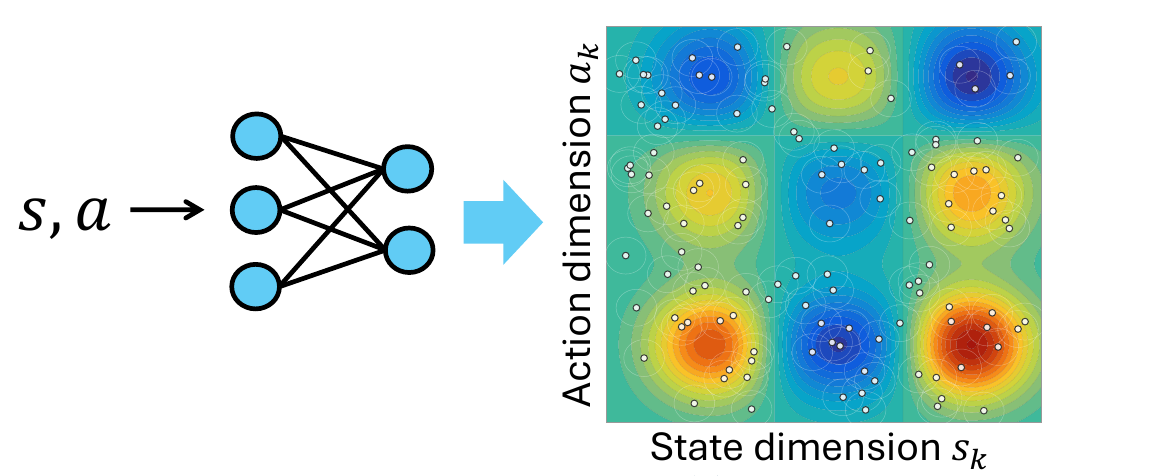}
    \caption{Joint Representation}
    \label{fig:concept-a}
  \end{subfigure}
  \begin{subfigure}{0.5\textwidth}
    \centering
    \includegraphics[width=\textwidth]{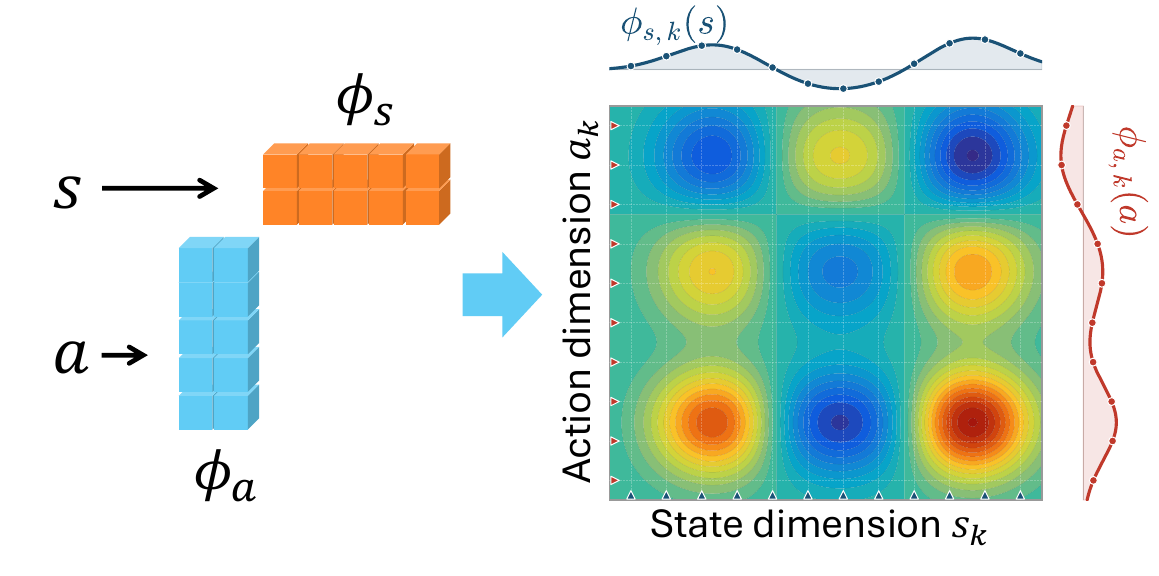}
    \caption{Factorized Representation}
    \label{fig:concept-b}
  \end{subfigure}
  \caption{Comparison of representation structures. Factorization constrains the joint score landscape to the intersection of independent state and action profiles.}
  \label{fig:concept}
\end{wrapfigure}
In such cases, a joint representation must relearn shared state-side structure across actions, whereas a factored representation can reuse the same state-side modes while learning how actions modulate them.

To this end, we introduce Factored Spectral Representations (FaStR), which treats the transition kernel as a three-way tensor and learns one encoder per mode (state, action, next-state) through a CANDECOMP/PARAFAC (CP) decomposition. Figure~\ref{fig:concept} illustrates the key difference. A joint encoder maps the concatenated pair $(s,a)$ directly into a state-action representation, so the learned score can vary freely over the full joint plane (Figure~\ref{fig:concept-a}). In this view, local peaks in the score landscape may be tied to specific state-action regions. FaStR instead learns separate state and action profiles, and combines them through a Hadamard product (Figure~\ref{fig:concept-b}). Thus, peaks and variations in the joint score must arise from the interaction of reusable state-side and action-side factors, encouraging sharing across rows and columns of the state-action plane. This structured sharing acts as an implicit regularizer and shifts the statistical burden from covering isolated state-action regions to estimating reusable factors, allowing the frozen feature to become useful for Bellman backups with fewer transitions when the dynamics match the factorization. Specifically, our contributions are as follows:
\begin{enumerate}[leftmargin=2em]
\item \textbf{A factored view of spectral RL representations.} We formulate the transition kernel as a three-mode object over state, action, and next state, and instantiate this view as a representation with a separate encoder for each of these three modes. The factored representation keeps these modes explicit while remaining compatible with contrastive training.
\item \textbf{A smaller sufficient-sample upper bound under CP structure.} We prove that the factored hypothesis class has a covering number that decomposes additively over the state, action, and next-state factors. Under CP realizability and the corresponding regularity conditions, FaStR obtains a smaller sufficient-sample upper bound than the joint class, with an exponent gap of $\min(d_s,d_a)$; under approximate CP structure, the guarantee becomes a bias-estimation tradeoff.
\item \textbf{Empirical evaluation and modular transfer.} On DM Control Suite~\citep{tassa2018dmcontrol} tasks across morphologies and action dimensions, we test when the factored representation outperforms a joint state-action encoder. FaStR obtains the highest mean return on all five tasks with $\dim\Acal\geq 21$ and matches or exceeds the joint encoder at lower action dimensions; Appendix~\ref{app:help-hurt} locates the boundary of this regime on seven further tasks. Under actuator shift, the state factor transfers intact and only the action factor is retrained.
\end{enumerate}
\section{Preliminaries}
\label{sec:prelim}

We consider a discounted MDP $\mathcal{M}=(\mathcal{S},\mathcal{A},P,r,\gamma)$ with continuous state space $\mathcal{S}$, continuous action space $\mathcal{A}$, transition kernel $P:\mathcal{S}\times\mathcal{A}\to\Delta(\mathcal{S})$, reward function $r:\mathcal{S}\times\mathcal{A}\to\mathbb{R}$, and discount factor $\gamma\in[0,1)$. The Q-function $Q^\pi(s,a)=\mathbb{E}_\pi[\sum_{t=0}^\infty \gamma^t r(s_t,a_t)\mid s_0=s,a_0=a]$ satisfies the Bellman equation
\begin{equation}\label{eq:bellman-main}
Q^\pi(s,a) = r(s,a) + \gamma \int_\mathcal{S} P(s'|s,a)\, V^\pi(s')\, ds',
\end{equation}
where $V^\pi(s)=\mathbb{E}_{a\sim\pi(\cdot|s)}[Q^\pi(s,a)]$.

\paragraph{Linear MDPs and the linear-feature view.}
A common simplification is to assume the transition kernel admits a finite-rank factorization. The linear MDP assumption~\citep{jin2020provably} states that there exist a feature map $\varphi:\mathcal{S}\times\mathcal{A}\to\mathbb{R}^d$, a kernel map
$\mu:\mathcal{S}\to\mathbb{R}^d$, and a reward parameter $\theta_r\in\mathbb{R}^d$ such that
\begin{equation}\label{eq:linear-mdp}
P(s'|s,a) = \varphi(s,a)^\top \mu(s'), \qquad r(s,a) = \varphi(s,a)^\top \theta_r.
\end{equation}
Under this assumption, the value function is linear in $\varphi$: for any policy $\pi$, $Q^\pi(s,a)=\varphi(s,a)^\top w^\pi$ for some $w^\pi\in\mathbb{R}^d$ that absorbs the integral against $V^\pi$. The complexity of value learning then depends on the feature dimension $d$ rather than the size of the state-action space, which enables provably efficient algorithms~\citep{jin2020provably,uehara2022representation}.
In practice $\varphi$ is unknown and must be learned from interaction data.

\paragraph{Spectral RL: learning transition-ratio scores.}
Spectral representation methods learn $\varphi$ and $\mu$ from transition samples using self-supervised contrastive objectives~\citep{gutmann2010noise,oord2018representation,zhang2022making,gao2025spectral}. A scoring function $h(s,a,s')=\varphi(s,a)^\top\mu(s')$ is trained to assign high values to true next-state samples $s'\sim P(\cdot|s,a)$ and low values to negatives drawn from a reference
distribution. At the population level, such objectives fit a transition-ratio score: the logit measures how plausible a candidate next state is under the transition kernel relative to the negative distribution. Existing deep spectral methods parameterize this score bilinearly, using a fused state-action feature on one side and a next-state feature on the other. The trained $\varphi$ can then be used as the linear-MDP feature for downstream value learning.

The contrastive template that FaStR will instantiate is the ranking-based perturbed noise contrastive estimation (RP-NCE) objective~\citep{zhang2022making,gao2025spectral}, in which negatives are obtained by perturbing true next states across a sequence of noise levels and the model is trained to rank a true (perturbed) positive above $K-1$ corrupted alternatives:
\begin{equation}\label{eq:nce-prelim}
\ell_{\text{NCE}} = -\frac{1}{MN}\sum_{m=1}^M\sum_{i=1}^N
\log\frac{\exp\bigl(\varphi_i^\top \mu(\tilde s'_{i,0};\beta_m)\bigr)}
{\sum_{j=0}^{K-1} \exp\bigl(\varphi_i^\top \mu(\tilde s'_{i,j};\beta_m)\bigr)},
\end{equation}
where $\varphi_i=\varphi(s_i,a_i)$, $\{\beta_m\}_{m=1}^M$ are noise levels along a variance-preserving (VP) schedule~\citep{song2021scorebased}, which keeps the marginal variance fixed across levels, and $\tilde s'_{i,j}$ denotes the $j$-th candidate next state corrupted at level $\beta_m$, with $j=0$ the true target and $j\geq 1$ negatives drawn from other transitions in the mini-batch (in-batch negatives). At the population level this ranking objective is consistent with maximum likelihood estimation as the number of negatives grows~\citep{ma2018noise}. Thus RP-NCE is best viewed as an estimator for a chosen transition-ratio score class; FaStR keeps this estimator but changes the score class from a fused bilinear to a mode-wise trilinear form. 

\paragraph{The full spectral RL pipeline.}
The contrastive loss is one component of the system. After training, the encoder $\varphi$ is used as a frozen representation: a critic $Q_\theta$ takes $\varphi(s,a)$ as input and is trained by standard temporal-difference (TD) losses without back-propagating through $\varphi$, and a policy network $\pi$ is trained on top of this frozen representation. This separation between representation learning (the contrastive loss) and value learning (TD on the frozen feature) is what makes the learned $\varphi$ usable as the linear-MDP feature in \citet{jin2020provably}. We work within this two-stage setup and locate our contribution in the structural form of the encoder $\varphi$.
\section{Method}
\label{sec:method}
\subsection{The FaStR Framework}
\label{sec:method-cp}
\paragraph{A tensor view of the transition kernel.}
The standard linear MDP merges $s$ and $a$ into a single index and factorizes the resulting matrix $P\in\mathbb{R}^{|\mathcal{S}\times\mathcal{A}|\times|\mathcal{S}|}$ via an SVD-like decomposition. This treats the two arguments uniformly, although the state specifies the configuration of the system while the action specifies the input applied to it. We separate the two arguments by viewing the transition kernel as a third-way tensor $\mathcal{T}\in\mathbb{R}^{|\mathcal{S}|\times|\mathcal{A}|\times|\mathcal{S}|}$ indexed by $s$, $a$, and $s'$.\footnote{The tensor view is stated in the finite case for intuition; Assumption~\ref{assump:factored} below is stated directly through feature maps on general $\mathcal{S}$ and $\mathcal{A}$.}

A standard general decomposition of such a tensor is the Tucker form~\citep{tucker1966some,kolda2009tensor} $P(s'|s,a)=\sum_{i,j,k}G_{ijk}\,\phi_{s,i}(s)\,\phi_{a,j}(a)\,m_k(s')$ with a core $G\in\mathbb{R}^{d_s\times d_a\times d_m}$. The CP (CANDECOMP/PARAFAC) decomposition~\citep{carroll1970analysis,harshman1970foundations,kolda2009tensor} is the special case where $G$ is superdiagonal, $G_{kkk}=1$, eliminating the core entirely; lower-rank or block-diagonal cores interpolate between CP and dense Tucker. Among these, FaStR uses the CP form; the case for this choice is made below in the context of contrastive training.

\begin{assumption}[CP-Factored MDP]\label{assump:factored}
There exist feature maps $\phi_s:\mathcal{S}\to\mathbb{R}^d$, $\phi_a:\mathcal{A}\to\mathbb{R}^d$, a kernel map $m:\mathcal{S}\to\mathbb{R}^d$, and a reward parameter $\theta_r\in\mathbb{R}^d$ such that for all $(s,a,s')$:
\begin{align}
P(s'|s,a) &= \bigl(\phi_s(s)\odot\phi_a(a)\bigr)^\top m(s')
  = \sum_{k=1}^d \phi_{s,k}(s)\,\phi_{a,k}(a)\,m_k(s'),
  \label{eq:cp}\\
r(s,a) &= \bigl(\phi_s(s)\odot\phi_a(a)\bigr)^\top \theta_r.
  \label{eq:reward}
\end{align}
We write $\psi(s,a):=\phi_s(s)\odot\phi_a(a)\in\mathbb{R}^d$ for the factored feature.
\end{assumption}

Each component $k$ in Equation~\eqref{eq:cp} is a rank-one term combining a state factor $\phi_{s, k}(s)$, an action factor $\phi_{a, k}(a)$, and a next-state factor $m_k\left(s^{\prime}\right)$, with the action factor scaling the coupling between the other two. CP fits well when actions scale shared state-dependent modes of motion, such as posture or balance in locomotion; it underfits when strong contact or limb-coupling effects entangle modes across actions. The CP form has scale invariances between $\phi_s$, $\phi_a$, and $m$, which are fixed by uniform norm bounds $B_\phi$ and $B_{m,\infty}$ throughout the analysis (Appendix~\ref{app:notation}); the implementation uses LayerNorm to stabilize the scales.

\paragraph{Latent factorization versus coordinate-wise separability.}
The CP interaction in Assumption~\ref{assump:factored} is diagonal in the learned latent features, not in the physical state or action variables. Both $\phi_s$ and $\phi_a$ are deep networks: each latent factor can depend on all state (or action) variables, and a single physical action can contribute to many latent factors. The diagonal core removes only the terms in which latent state factor $i$ meets latent action factor $j\neq i$. The assumption therefore requires that a low-rank CP model approximate the transition well in the learned latent space; it does not require independent physical actuators. Couplings that the encoders can reorganize are absorbed by them, and the remainder is the CP approximation error $\epsilon_{\mathrm{CP}}$. The diagonal form trades this approximation error for a smaller function class, and Proposition~\ref{prop:finite-sample} states the tradeoff.

\paragraph{Factored representation and Q-linearity.}
FaStR learns a feature that is useful for control, not only for predicting
the next state.  Let
$\psi(s,a):=\phi_s(s)\odot\phi_a(a)$.  Under the CP-factored model,
the transition and reward share this same feature:
$P(s'|s,a)=\psi(s,a)^\top m(s')$ and
$r(s,a)=\psi(s,a)^\top\theta_r$.  Therefore, a Bellman backup does not
leave the span of $\psi$: for any policy $\pi$, $Q^\pi(s,a)$ is derived as:
\begin{equation}\label{eq:Q-linear}
r(s,a)+\gamma\!\int_{\mathcal S}P(s'|s,a)V^\pi(s')ds'
= \psi(s,a)^\top\!\left(\theta_r+\gamma\!\int_{\mathcal S}m(s')V^\pi(s')ds'\right)
=: \psi(s,a)^\top w^\pi .
\end{equation}
Thus the downstream linear critic is not an additional assumption: once the
transition and reward admit the CP form, the corresponding $Q$-function is
linear in the learned factored feature.  A derivation with norm bounds is
given in Appendix~\ref{app:proof-bilinear-Q}.

\paragraph{Learning a trilinear transition-ratio score.}
We learn the factors by viewing contrastive spectral learning as conditional
ratio estimation. For an anchor $(s_i,a_i)$, the observed next state is the
positive candidate and perturbed in-batch next states are negatives. InfoNCE~\citep{oord2018representation}
turns this into a classification problem: among one true next state and several
distractors, the model is trained to assign the highest energy to the true
candidate. At the population optimum, differences between these energies
estimate the log ratio between the transition law at the anchor and the
negative candidate distribution. RP-NCE repeats this classification across noise levels $\beta$, giving a multi-scale ratio target for the next-state side. In standard spectral RL, this energy is bilinear: a fused state-action feature is matched with a $\beta$-conditioned next-state feature. FaStR instead chooses the CP trilinear energy
\[
f_\Theta(s,a,s';\beta)
=
\sum_{k=1}^d
\phi_{s,\Theta,k}(s)\,\phi_{a,\Theta,k}(a)\,
m_{\Theta,k}(s';\beta)
=
\bigl(\phi_{s,\Theta}(s)\odot\phi_{a,\Theta}(a)\bigr)^\top
m_\Theta(s';\beta).
\]
Each coordinate is a ratio channel: the state factor selects a state-dependent mode, the action factor gates it, and the next-state factor scores the candidate at noise level $\beta$. Substituting this energy into the RP-NCE template~\eqref{eq:nce-prelim} gives
\begin{equation}\label{eq:nce}
\ell_{\mathrm{FaStR}}
=
-\frac{1}{MN}\sum_{b=1}^{M}\sum_{i=1}^{N}
\log
\frac{
\exp\!\left(f_\Theta(s_i,a_i,\tilde s'_{i,0};\beta_b)\right)}
{
\sum_{j=0}^{K-1}
\exp\!\left(f_\Theta(s_i,a_i,\tilde s'_{i,j};\beta_b)\right)
}.
\end{equation}
The learned state-action feature is the anchor side of this energy,
$\psi_\Theta(s,a)=\phi_{s,\Theta}(s)\odot\phi_{a,\Theta}(a)$. The downstream
RFF critic then kernelizes the frozen $\psi_\Theta$ for value learning, while the mode-wise geometry of $\psi_\Theta$ is fixed by the transition-ratio energy. After this representation stage, we freeze $\psi_\Theta$, train the actor and critic on top of it, and use the same downstream control pipeline as prior spectral RL; the structural change is in how transition data shape the state-action feature. Algorithm~\ref{alg:factored-sr} gives the per-step update.
\paragraph{CP vs. Tucker cores.}
The CP form is the diagonal member of a broader trilinear transition-ratio score class, written as $h(s,a,s')=\phi_s(s)^\top M(s')\phi_a(a)$, with $M(s')$ generated by the next-state mode and restricted to be diagonal under CP. Other parameterizations correspond to less constrained interaction matrices: a dense $d_s\times d_a$ matrix gives full Tucker, and intermediate cases (low-rank $M(s')=U(s')V(s')^\top$, block-diagonal cores) interpolate between CP and dense Tucker. Such factorizations have been used in prior bilinear MDP methods trained with regression or maximum-likelihood objectives~\citep{yang2020reinforcement, du2021bilinear, lin2024because}, where the gradient on $M(s')$ is full-rank per anchor and large interaction matrices can be trained with enough data.

Under contrastive training the gradient structure is different. The chain rule on $h_{ij}=\phi_s(s_i)^\top M(s'_j)\phi_a(a_i)$ gives
\begin{equation}\label{eq:rank-one-grad}
\nabla_{\mathrm{vec}(M(s'_j))}\,\ell_{\text{NCE}}
\;\in\; \mathrm{span}\bigl\{\phi_a(a_i)\otimes\phi_s(s_i)\bigr\},
\end{equation}
a one-dimensional subspace per anchor. A mini-batch of $N$ anchors therefore covers at most $\min(N,\dim(M))$ directions in the parameter space of $M(s')$ in a single update. This rank-one structure follows from the conditional bilinear form of the trilinear score and is the same for any parameterization of $M$: low-rank or block-diagonal cores reduce the parameter count below $d_s d_a$, but each anchor still contributes a rank-one direction projected onto the corresponding submanifold. For dense Tucker, $\dim(M)=d_s d_a$, so per-batch coverage requires $N\geq d_s d_a$, which scales quadratically in the per-mode dimension and exceeds typical batch sizes. The CP form has $\dim(M)=d$, so the linear requirement $N\geq d$ is met by standard mini-batch sizes.

Structured cores cannot satisfy encoder capacity, NCE validity ($N\geq d_s d_a$), and per-batch gradient coverage simultaneously at standard batch sizes; in the controlled ablation of Appendix~\ref{app:nce-ablation}, no Tucker rank within this constraint matches CP under matched compute.

\subsection{Provable Representation Efficiency}
\label{sec:complexity}

Section~\ref{sec:method-cp} chose CP among bilinear parameterizations on training-dynamics grounds. The matching statistical question is whether CP also makes the representation class easier to learn from a finite sample. We answer it in three steps. We first bound the size of the factored class relative to the joint class (Theorem~\ref{thm:complexity}). We then convert this bound into a finite-sample bound on representation error (Proposition~\ref{prop:finite-sample}). Finally, we invert the bound to compare the sample sizes that suffice for each class to reach a target error (Proposition~\ref{prop:sample-sep}). The dimension dependence of the final result matches the empirical pattern in Section~\ref{sec:experiments}.

We work with two scoring classes. The factored class $\Gcal_{\text{fac}}$ contains $g(s,a,s')=(\phi_s(s)\odot\phi_a(a))^\top m(s')$ with $(\phi_s,\phi_a,m)\in\Fcal_s\times\Fcal_a\times\Fcal_m$; the joint class
$\Gcal_{\text{joint}}$ contains $g(s,a,s')=\varphi(s,a)^\top\mu(s')$ with $\varphi\in\Fcal_{sa}$, $\mu\in\Fcal_m$. Norm bounds and the convention for sup-$\ell_2$ covering numbers are in Appendix~\ref{app:notation}.

\paragraph{Step 1: bounding the size of the factored class.}
A monolithic encoder takes each $(s,a)$ pair as a single input, and the size of its hypothesis class scales with the combined state-action dimension. The CP parameterization splits the score into three networks: a state factor, an action factor, and a kernel encoder, coupled through a diagonal trilinear interaction (the Hadamard form of Eq.~\ref{eq:cp}). The size of the resulting class is then controlled by the three component networks separately rather than by their product. Covering numbers, a standard measure of the size of a function class, make this precise.

\begin{theorem}[Covering-number decomposition]\label{thm:complexity}
For every $\epsilon>0$,
\[
\log\Ncal(\Gcal_{\emph{fac}},\epsilon)
\;\le\;
\log\Ncal\!\left(\Fcal_s,\tfrac{\epsilon}{3B_{m,\infty}B_\phi}\right)
+
\log\Ncal\!\left(\Fcal_a,\tfrac{\epsilon}{3B_{m,\infty}B_\phi}\right)
+
\log\Ncal\!\left(\Fcal_m,\tfrac{\epsilon}{3B_\phi^2}\right).
\]
If $\Fcal_{sa}$ contains every Hadamard product $\phi_s\odot\phi_a$ with $(\phi_s,\phi_a)\in\Fcal_s\times\Fcal_a$, then $\Gcal_{\emph{fac}}\subseteq\Gcal_{\emph{joint}}$, and the inclusion holds for both covering numbers and Rademacher complexities at every sample size.
\end{theorem}
\begin{proof}
See Appendix~\ref{app:proof-complexity}.
\end{proof}
For Lipschitz encoders on compact domains, this gives an exponent $D_{\max}=\max(d_s,d_a)$ for the factored class versus $D_{\mathrm{sum}}=d_s+d_a$ for the joint class, leaving a gap of $\min(d_s,d_a)$ in the rate exponent that is formalized as a comparison of sufficient sample sizes in Proposition~\ref{prop:sample-sep}.
Once $\psi=\phi_s\odot\phi_a$ is frozen, TD on $\psi$ satisfies the linear-MDP sample-complexity guarantees of \citet{jin2020provably}, so a smaller bound on the representation error at a given $n$ becomes a smaller bound on the return gap in policy learning; Appendix~\ref{app:operational} reads the bound on Humanoid.

\paragraph{Step 2: from class size to representation error.}
A smaller class size matters only if it leads to a smaller generalization gap. The target we analyze is the $L_2$ representation objective of \citet{ren2022spectral}, the squared $L^2(\nu)$ distance from the true transition density,
\[
L(g) := \E_\rho\!\left[\int_\Scal\bigl(P(s'\mid s,a)-g(s,a,s')\bigr)^2 d\nu(s')\right].
\]
We analyze this objective rather than NCE itself because finite-sample guarantees are tractable for the $L_2$ surrogate at the population level, while the contrastive loss serves as its practical proxy at training time~\citep{ren2022spectral}.

\begin{proposition}[Finite-sample generalization]\label{prop:finite-sample}
Let $\Gcal\in\{\Gcal_{\emph{fac}},\Gcal_{\emph{joint}}\}$ with 
$|g|\le B_g$, and let $\hat g_n$ be an empirical minimizer of the 
centered objective on $n$ i.i.d.\ transitions. With probability at least 
$1-\delta$,
\[
L(\hat g_n) \;\le\; \underbrace{L(g^*_\Gcal)}_{\text{approximation}}
\;+\; \underbrace{8\,\Rcal_n(\Gcal) + 4\,\Rcal_n(\Hcal_\Gcal)}_{\text{estimation}}
\;+\; O\!\left((B_g+B_g^2)\sqrt{\tfrac{\log(1/\delta)}{n}}\right),
\]
where $\Hcal_\Gcal:=\{(s,a)\mapsto\int g^2(s,a,u)\,d\nu(u):g\in\Gcal\}$ 
is the integrated quadratic class arising from expanding the squared 
loss. The quadratic class satisfies 
$\Ncal(\Hcal_\Gcal,\epsilon)\le\Ncal(\Gcal,\epsilon/(2B_g))$, so its 
Rademacher complexity inherits the rate of $\Gcal$ itself.
\end{proposition}
\begin{proof}
See Appendix~\ref{app:proof-finite-sample}.
\end{proof}
The bound has two terms with opposite scaling. The approximation term $L(g^*_\Gcal)$ measures whether the class contains the true dynamics; under exact CP realizability and the inclusion of Theorem~\ref{thm:complexity}, both classes contain the truth and this term is zero for both. Under approximate CP, the factored class has an irreducible bias $\epsilon_{\mathrm{CP}}^2:=\inf_{g\in\Gcal_{\text{fac}}}L(g)$ that the joint class does not have. The estimation term is smaller for the factored class: both Rademacher terms inherit the smaller exponent $D_{\max}$ from Theorem~\ref{thm:complexity}, so the factored certificate decays faster in $n$.

\paragraph{Step 3: from representation error to sample size.}
The previous two steps bound representation error at a fixed sample size. The inverse question is how many samples each class needs to reach the same target error.

\begin{proposition}[Sufficient-sample-size separation]\label{prop:sample-sep}
Assume exact CP realizability, the inclusion condition of Theorem~\ref{thm:complexity}, and $L$-Lipschitz encoder classes on compact domains with $D_{\max}\geq 3$. Proposition~\ref{prop:finite-sample} then gives certificates of the form
$L(\hat g_n^{\mathrm{fac}})\le C_{\mathrm{fac}}(\delta)\,n^{-1/D_{\max}}$ and 
$L(\hat g_n^{\mathrm{joint}})\le C_{\mathrm{joint}}(\delta)\,n^{-1/D_{\mathrm{sum}}}$, with $C_{\mathrm{fac}}, C_{\mathrm{joint}}$ depending polynomially on the norm bounds, Lipschitz constants, and feature dimension, and logarithmically on $1/\delta$. Let 
$n_{\mathrm{fac}}^{\mathrm{suff}}(\epsilon)$ and 
$n_{\mathrm{joint}}^{\mathrm{suff}}(\epsilon)$ be the smallest $n$ at which each certificate guarantees error at most $\epsilon^2$. Then
\[
\frac{n_{\mathrm{joint}}^{\mathrm{suff}}(\epsilon)}
     {n_{\mathrm{fac}}^{\mathrm{suff}}(\epsilon)}
\;=\;\Theta\!\left(\epsilon^{-2\min(d_s,d_a)}\right) \quad \text{as } \epsilon\to 0.
\]
\end{proposition}
\begin{proof}
See Appendix~\ref{app:proof-sample-sep}.
\end{proof}
Under CP realizability and the corresponding regularity conditions, FaStR therefore obtains a smaller sufficient-sample upper bound, with the exponent gap equal to $\min(d_s,d_a)$; Appendix~\ref{app:proof-sample-sep} details the scope of this upper-bound comparison. This gap is meaningful when the smaller of the state and action dimensions is nontrivial and the CP bias is small; otherwise the approximate-CP certificate $L(\hat g_n^{\mathrm{fac}}) \lesssim \epsilon_{\mathrm{CP}}^2 + C_{\mathrm{fac}}(\delta)n^{-1/D_{\max}}$ can be dominated by bias. We test these two conditions in Section~\ref{sec:experiments} and Appendix~\ref{app:cp-misfit}. Appendix~\ref{app:lsvi-conditional} further gives a conditional LSVI-UCB implication: if the learned feature satisfies approximate Bellman linearity, the factored structure can reduce the representation cost of obtaining features suitable for optimism-driven exploration.
\section{Experiments}
\label{sec:experiments}

We compare FaStR against the monolithic encoder under an otherwise identical pipeline on locomotion benchmarks (Section~\ref{sec:exp-singletask}), and test whether the separable state factor enables adaptation under actuator shift (Section~\ref{sec:exp-transfer}).
\subsection{Setup}
\label{sec:exp-setup}
We evaluate on DM Control Suite~\citep{tassa2018dmcontrol} locomotion tasks that vary in action dimensionality and morphology, with $\dim(\Acal)\in\{6,12,21,38\}$ and shared proprioceptive observations.
To attribute performance differences to the encoder structure, FaStR and the monolithic baseline CTRL-SR~\citep{gao2025spectral} share the same TD3 backbone, hidden width, representation dimension, learning rates, noise schedule, and training budget, with no task-specific tuning. The only varying component is the representation: FaStR's factored encoders $(\phis,\phia,m)$ versus CTRL-SR's joint encoder $\varphi([s;a])$. We align hidden width rather than parameter count, and Appendix~\ref{app:width-ablation} repeats the comparison at three widths. Figure~\ref{fig:single-task} compares FaStR against CTRL-SR. Table~\ref{tab:full-results} adds Diff-SR~\citep{shribak2024diffusion}, a spectral baseline that replaces the contrastive objective with a diffusion next-state model on the same TD3 backbone and training hyperparameters; the model-free methods TD7~\citep{fujimoto2023td7} and SAC~\citep{haarnoja2018soft}; and the model-based TD-MPC2~\citep{hansen2024tdmpc2} in its official 5M-parameter single-task configuration. TD7 uses its authors' recommended hyperparameters, and SAC(3-layer) uses the hidden width of the spectral methods. All methods share the environment-interaction budget ($10^6$ frames), action repeat, evaluation protocol, and random seeds, and all results are reported over 8 seeds. Full hyperparameters are in Appendix~\ref{app:impl}.

\paragraph{Task selection.} We evaluate 5 DMControl domains with action dimensions of at least 6, spanning action dimensions from 6 to 38. Within these domains, we include representative standing, walking, and running tasks. Under our computational budget, this provides broad coverage of action dimensions and morphologies without evaluating every reward variant. The task set was fixed before the experiments, and neither training stability nor performance influenced the selection. Seven additional official tasks with $\dim\Acal\le 5$ are reported in Appendix~\ref{app:complementary-tasks}, and Table~\ref{tab:task-coverage} lists every evaluated task.
\subsection{Sample Efficiency on Locomotion Benchmarks}
\label{sec:exp-singletask}
\begin{figure}[h]
  \centering
  \includegraphics[width=\textwidth]{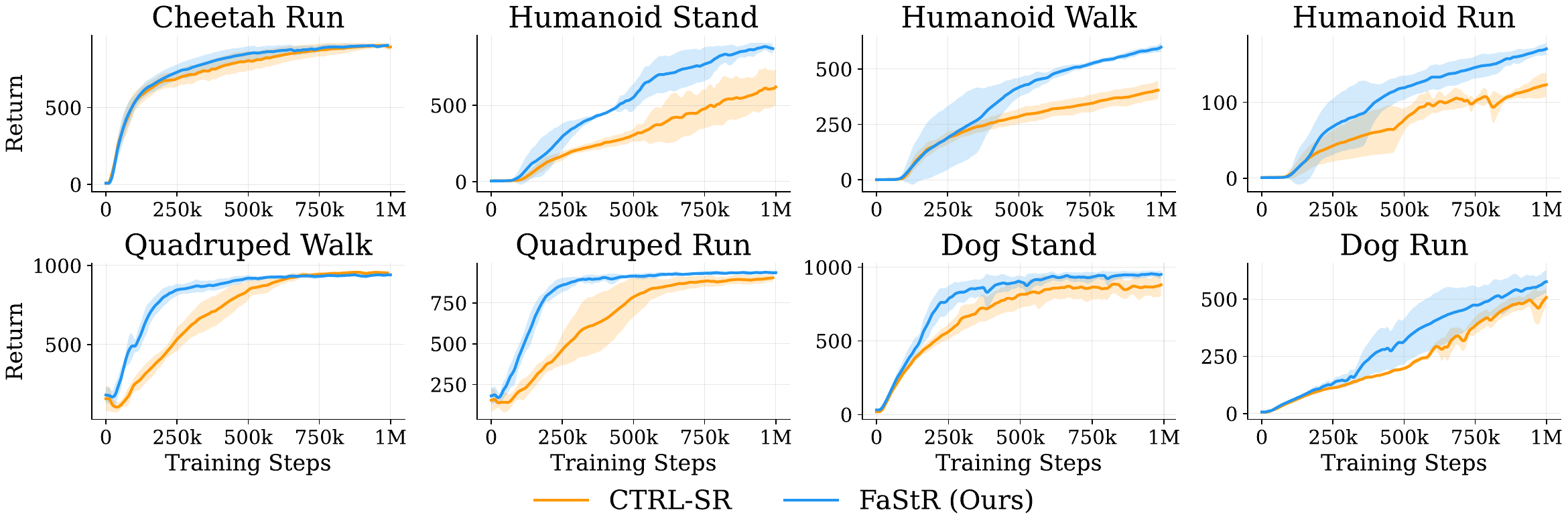}
  \caption{Learning curves across 8 DM Control Suite tasks spanning four morphologies. FaStR's gains are largest on tasks whose dynamics admit a CP factorization, consistent with the bias-estimation view. Shaded: $\pm1$~std across 8 seeds.}
  \label{fig:single-task}
\end{figure}
On the Humanoid tasks, where torso-limb morphology admits a CP factorization (actions scaling shared postural modes in the learned latent space), FaStR exceeds the controlled baseline CTRL-SR by 106, 197, and 56 points on Stand, Walk, and Run, each margin larger than the across-seed standard deviation of either method. On Dog-Stand, FaStR reaches $955\pm20$ against $870\pm70$. On Dog-Run, where four limbs move in coordination and ground contact couples state and action more strongly, the improvement over the joint encoder is smaller than on the other high-dimensional tasks ($538\pm49$ vs.\ $505\pm97$). FaStR also exceeds Diff-SR on all five tasks with $\dim\Acal\geq21$ (Table~\ref{tab:full-results}), so the gain is not specific to the contrastive objective. TD-MPC2, which learns a compact world model for planning, obtains the highest return on Cheetah-Run and Walker-Walk, while FaStR obtains the highest mean return on all five $21$- and $38$-dimensional tasks: full dynamics modeling pays off in lower-dimensional control, and a low-rank transition factorization is the more effective target on this suite as the action dimension grows.

On the Quadruped tasks, FaStR reaches CTRL-SR's converged return earlier in training: on Quadruped-Run it then exceeds it ($941\pm14$ vs.\ $910\pm18$), and on Quadruped-Walk the final returns coincide ($952\pm10$ vs.\ $954\pm8$). The estimation advantage on this family appears as faster convergence, the rate at which the representation becomes usable for value learning. On Cheetah-Run, Walker-Walk, and Walker-Run ($\dim\Acal{\le}6$), FaStR matches the spectral baselines, the regime predicted by Proposition~\ref{prop:sample-sep} in which small $\min(d_s,d_a)$ gives a small estimation advantage in the certificate.
\begin{table}[h]
\centering
\caption{Final-10-evaluation mean return on DM Control Suite tasks with proprioceptive observations, averaged across 8 seeds ($\pm1$~std). Bold: best mean return in each column.}
\label{tab:full-results}
\footnotesize
\setlength{\tabcolsep}{2.5pt}
\renewcommand{\arraystretch}{0.95}
\resizebox{\textwidth}{!}{%
\begin{tabular}{@{}l ccc cc ccc cc@{}}
\toprule
& \multicolumn{3}{c}{$\dim\Acal=6$}
& \multicolumn{2}{c}{$\dim\Acal=12$}
& \multicolumn{3}{c}{$\dim\Acal=21$}
& \multicolumn{2}{c}{$\dim\Acal=38$} \\
\cmidrule(lr){2-4} \cmidrule(lr){5-6} \cmidrule(lr){7-9} \cmidrule(lr){10-11}
& {\scriptsize Cheetah Run} & {\scriptsize Walker Walk} & {\scriptsize Walker Run}
& {\scriptsize Quad.\ Walk} & {\scriptsize Quad.\ Run}
& {\scriptsize Hum.\ Stand} & {\scriptsize Hum.\ Walk} & {\scriptsize Hum.\ Run}
& {\scriptsize Dog Stand} & {\scriptsize Dog Run} \\
\midrule
SAC(2-layer) & $371\pm98$ & $484\pm12$ & $665\pm23$ & $717\pm54$ & $777\pm52$  & $223\pm84$  & $170\pm10$  & $\;\,40\pm65$ & $500\pm9$   & $\;\,39\pm66$ \\
SAC(3-layer) & $885\pm41$ & $961\pm27$ & $792\pm26$ & $872\pm81$ & $884\pm105$ & $428\pm64$  & $229\pm142$ & $\;\,88\pm30$ & $577\pm11$  & $\;\,67\pm34$ \\
TD7          & $909\pm31$ & $979\pm1$  & $813\pm17$ & $952\pm3$  & $875\pm114$ & $660\pm230$ & $334\pm208$ & $\;\,80\pm45$ & $840\pm128$ & $172\pm51$ \\
TD-MPC2      & $\mathbf{933\pm5}$ & $\mathbf{988\pm4}$ & $830\pm12$ & $953\pm11$ & $919\pm7$ & $673\pm13$ & $589\pm37$ & $161\pm24$ & $799\pm40$ & $181\pm33$ \\
Diff-SR      & $910\pm4$  & $978\pm2$  & $\mathbf{833\pm15}$ & $953\pm5$ & $903\pm27$ & $761\pm129$ & $440\pm93$ & $107\pm35$ & $909\pm27$ & $310\pm94$ \\
CTRL-SR      & $898\pm14$ & $978\pm1$  & $811\pm25$ & $\mathbf{954\pm8}$ & $910\pm18$ & $752\pm94$ & $394\pm37$ & $118\pm13$ & $870\pm70$ & $505\pm97$ \\
FaStR (Ours) & $899\pm13$ & $975\pm3$  & $814\pm22$ & $952\pm10$ & $\mathbf{941\pm14}$ & $\mathbf{858\pm26}$ & $\mathbf{591\pm15}$ & $\mathbf{174\pm19}$ & $\mathbf{955\pm20}$ & $\mathbf{538\pm49}$ \\
\bottomrule
\end{tabular}%
}
\end{table}
\paragraph{Mechanism check.} In a controlled ablation (Appendix~\ref{app:nce-ablation}), separate encoders feeding a joint MLP improve only slightly on the monolithic baseline, and a full Tucker core trains worse under NCE because of the rank-deficient gradient of Section~\ref{sec:method-cp}: the CP interaction is the source of the gain. The offline CP-misfit diagnostic (Appendix~\ref{app:cp-misfit}) shows that gains track structural alignment with the dynamics, not action dimension alone. The ranking on the Humanoid tasks is unchanged across encoder widths of 256, 512, and 1024 (Appendix~\ref{app:width-ablation}). On seven further tasks with $\dim\Acal\leq5$, FaStR is comparable to CTRL-SR on five and below it on Hopper-Hop and Acrobot-Swingup (Appendix~\ref{app:complementary-tasks}).
\subsection{Modular Reuse of $\phis$ under Actuator Shift}
\label{sec:exp-transfer}
\begin{figure}[h]
  \centering
  \includegraphics[width=\textwidth]{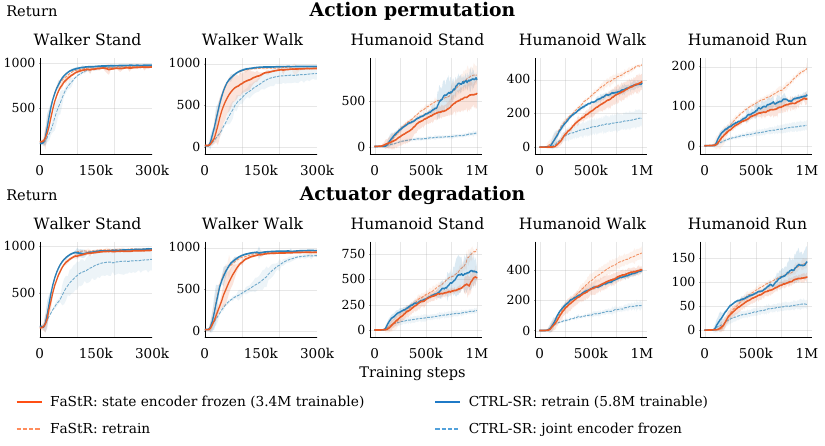}
\caption{Architectural diagnostic under actuator shift. Top: action permutation; bottom: actuator degradation. Walker ($\dim\Acal{=}6$) is a pre-shift parity check; Humanoid ($\dim\Acal{=}21$) gives the diagnostic. \textbf{What is preserved} (red solid vs.\ blue dashed): under the same freeze-and-adapt protocol, FaStR's $\phis$ remains usable across the shift while CTRL-SR's $\varphi$ does not. \textbf{At what cost} (red solid vs.\ blue solid): on Humanoid, FaStR with 18\% of its parameters held fixed in $\phis$ reaches the same final return as CTRL-SR retrained from scratch with all parameters trainable; trainable parameter counts for the two solid curves are reported in the legend. Orange dashed (FaStR retrained from scratch) is a reference upper bound. Shaded: $\pm1$~std across 8~seeds. See Appendix~\ref{app:impl} for parameter setting.}
  \label{fig:transfer}
\end{figure}
We test transfer in a setting where the physical dynamics of the system are unchanged but the action-to-torque interface is rewritten. The CP factorization predicts what each architecture preserves under such a change. Under FaStR, $\phis$ depends only on the state-side dynamics of the body and should remain valid when only the action-to-torque map shifts; under a joint encoder, state and action are coupled through every parameter, so any change to the interface leaks throughout the encoder and no component is guaranteed to remain valid. The factored representation therefore admits an option a joint encoder does not: keep $\phis$, adapt only $\phia$.

We instantiate two such rewrites. Actuator degradation applies independent per-joint gain scaling and polarity inversion. Action permutation randomly reassigns joint channel indices. The transition over true torques and the reward are both unchanged; only the path from policy output to actuation is rewritten. Source and target therefore share the same task, reward, and body, and differ only in the actuator interface, so any retraining cost reflects what the representation failed to preserve. For each architecture we freeze the largest state-dependent component ($\phis$ for FaStR, the joint encoder $\varphi$ for CTRL-SR) and reinitialize everything downstream. Walker is a parity check on matched pre-shift returns; Humanoid is the diagnostic.
\paragraph{What is preserved, and at what cost.} On a CP-factored MDP, which component to keep across an actuator shift is determined by the representation, not a hyperparameter the practitioner has to choose. A monolithic encoder gives no such handle: the only options are to retrain everything or to freeze arbitrary subsets, neither of which uses structural information about the shift. Reusing the frozen $\phis$ costs no final return: on Humanoid-Walk, FaStR with $\phis$ frozen matches CTRL-SR retrained from scratch with all parameters trainable (393.7 vs.\ 394.1 under permutation, 401.8 vs.\ 399.2 under degradation; Figure~\ref{fig:transfer}).
\section{Related Work}
\label{sec:related}
\paragraph{State and temporal representations.} Many RL methods ask what information a representation should preserve, rather than modeling the transition itself. Successor representations describe which states a policy is likely to visit in the future, and successor features make this information reusable across related reward functions~\citep{dayan1993improving,barreto2017successor}. Forward--backward representations describe whether a state--action pair can reach a future state~\citep{touati2021learning}. Laplacian and proto-value methods capture the geometry of the transition graph~\citep{mahadevan2005proto,wu2019laplacian}, while bisimulation methods group states with similar rewards and transition behavior~\citep{ferns2004metrics,zhang2021learning}. Contrastive and predictive methods learn state features by matching related observations or predicting future ones~\citep{laskin2020curl,stooke2021decoupling,schwarzer2021data}. TACO and SALE incorporate actions more explicitly, but still use the learned representations mainly to improve policy or value learning~\citep{zheng2023taco,fujimoto2023td7}.

\paragraph{Transition representations.} Another line of work learns representations directly from the one-step transition. Linear MDP and bilinear-class theory show that low-dimensional transition features can support value estimation and exploration, while FLAMBE and \citet{uehara2022representation} study how such features can be recovered from data~\citep{jin2020provably,du2021bilinear,agarwal2020flambe,uehara2022representation}. Spectral RL provides practical algorithms for learning these representations. CTRL and SPEDER use contrastive objectives~\citep{zhang2022making,ren2022spectral}. LV-Rep models transitions through latent variables~\citep{ren2023latent}, Diff-SR learns from noise-perturbed next states~\citep{shribak2024diffusion}, and CTRL-SR combines perturbation with ranking-based contrastive learning~\citep{gao2025spectral}. Existing methods learn a joint representation of the state--action pair. FaStR instead studies whether transition dynamics admit a compact low-rank structure with separate state and action factors.

\paragraph{Low-rank structure in RL.} Low-rank assumptions have also been introduced elsewhere in reinforcement learning. \citet{ni2023learning} model the transition as a Tucker tensor learned from batch data for abstraction and policy evaluation. Other work instead imposes low-rank structure on value functions or continuous-action critics and policies~\citep{mahajan2021tesseract,lobel2023qfunctionals,rozada2024tensor,rozada2025tensor}. FaStR studies a different question: whether a CP structure already exists in the transition representation itself.

\paragraph{Dynamics representations for planning.} A different line of work also learns from transition dynamics, but aims to reconstruct the dynamics themselves. Early model-based methods learn predictive dynamics models for planning or policy optimization~\citep{janner2019trust,hafner2020dream}. More recent world-model methods such as DreamerV3 and TD-MPC2 learn latent dynamics models for imagination or model predictive control~\citep{hafner2023mastering,hansen2024tdmpc2}. FaStR instead learns only the transition representation needed for value learning, without reconstructing the full dynamics or performing planning.
\section{Discussion and Conclusion}
\label{sec:conclusion}
We introduce FaStR, a spectral representation learning framework that decomposes the transition kernel as a three-way tensor with state, action, and next-state factors, estimated through a trilinear structured contrastive objective. Under CP realizability and the corresponding regularity conditions, the factored class obtains a smaller sufficient-sample upper bound: transition samples constrain reusable state and action factors rather than isolated state-action pairs, so the frozen feature can become useful for Bellman backups earlier in training when the CP structure matches the dynamics. The same separation also supports modular adaptation: when the action-to-torque mapping shifts, the state factor transfers intact while only the action factor needs to be retrained. FaStR is most useful when the transition compresses into a small number of reusable latent modes and the state or action dimension is large enough for this reduction to matter; Appendix~\ref{app:help-hurt} maps both conditions empirically. The broader principle is that representations should structurally mirror the physical factorization of future embodied systems into sensors, actuators, and the dynamics connecting them.

\begin{ack}
This work was supported by the National Science Foundation under CAREER Award No.~2622141. Any opinions, findings, and conclusions or recommendations expressed in this material are those of the authors and do not necessarily reflect the views of the National Science Foundation. Competing interests: The authors declare no competing interests.
\end{ack}

\IfFileExists{references.bib}{%
  \bibliographystyle{plainnat}
  \bibliography{references}

\begin{thebibliography}{49}
\bibitem[Agarwal et~al.(2020)]{agarwal2020flambe}
Alekh Agarwal, Sham Kakade, Akshay Krishnamurthy, and Wen Sun. FLAMBE: Structural complexity and representation learning of low rank MDPs. In Advances in Neural Information Processing Systems, 2020.

\bibitem[Barreto et~al.(2017)]{barreto2017successor}
André Barreto, Will Dabney, Rémi Munos, Jonathan J. Hunt, Tom Schaul, Hado van Hasselt, and David Silver. Successor features for transfer in reinforcement learning. In Advances in Neural Information Processing Systems, 2017.

\bibitem[Bartlett and Mendelson(2002)]{bartlett2002rademacher}
Peter L. Bartlett and Shahar Mendelson. Rademacher and Gaussian complexities: Risk bounds and structural results. Journal of Machine Learning Research, 3:463--482, 2002.

\bibitem[Carroll and Chang(1970)]{carroll1970analysis}
J. Douglas Carroll and Jih-Jie Chang. Analysis of individual differences in multidimensional scaling via an N-way generalization of ``Eckart-Young'' decomposition. Psychometrika, 35(3):283--319, 1970. doi: 10.1007/BF02310791.

\bibitem[Dayan(1993)]{dayan1993improving}
Peter Dayan. Improving generalization for temporal difference learning: The successor representation. Neural Computation, 5(4):613--624, 1993.

\bibitem[Du et~al.(2021)]{du2021bilinear}
Simon S. Du, Sham M. Kakade, Jason D. Lee, Shachar Lovett, Gaurav Mahajan, Wen Sun, and Ruosong Wang. Bilinear classes: A structural framework for provable generalization in RL. In Proceedings of the International Conference on Machine Learning, 2021.

\bibitem[Dudley(1967)]{dudley1967sizes}
R. M. Dudley. The sizes of compact subsets of Hilbert space and continuity of Gaussian processes. Journal of Functional Analysis, 1(3):290--330, 1967. doi: 10.1016/0022-1236(67)90017-1.

\bibitem[Ferns et~al.(2004)]{ferns2004metrics}
Norm Ferns, Prakash Panangaden, and Doina Precup. Metrics for finite Markov decision processes. In Proceedings of the 20th Conference on Uncertainty in Artificial Intelligence, pages 162--169, 2004.

\bibitem[Fujimoto et~al.(2018)]{fujimoto2018addressing}
Scott Fujimoto, Herke van Hoof, and David Meger. Addressing function approximation error in actor-critic methods. In Proceedings of the International Conference on Machine Learning, 2018.

\bibitem[Fujimoto et~al.(2023)]{fujimoto2023td7}
Scott Fujimoto, Wei-Di Chang, Edward J. Smith, Shixiang Shane Gu, Doina Precup, and David Meger. For SALE: State-action representation learning for deep reinforcement learning. In Advances in Neural Information Processing Systems, 2023.

\bibitem[Gao et~al.(2025)]{gao2025spectral}
Chenxiao Gao, Haotian Sun, Na Li, Dale Schuurmans, and Bo Dai. Spectral representation-based reinforcement learning. arXiv preprint arXiv:2512.15036, 2025.

\bibitem[Gutmann and Hyv{\"a}rinen(2010)]{gutmann2010noise}
Michael Gutmann and Aapo Hyvärinen. Noise-contrastive estimation: A new estimation principle for unnormalized statistical models. In Yee Whye Teh and Mike Titterington, editors, Proceedings of the International Conference on Artificial Intelligence and Statistics, volume 9 of Proceedings of Machine Learning Research, pages 297--304. PMLR, 2010.

\bibitem[Haarnoja et~al.(2018)]{haarnoja2018soft}
Tuomas Haarnoja, Aurick Zhou, Pieter Abbeel, and Sergey Levine. Soft actor-critic: Off-policy maximum entropy deep reinforcement learning with a stochastic actor. In Proceedings of the International Conference on Machine Learning, 2018.

\bibitem[Hafner et~al.(2020)]{hafner2020dream}
Danijar Hafner, Timothy Lillicrap, Jimmy Ba, and Mohammad Norouzi. Dream to control: Learning behaviors by latent imagination. In International Conference on Learning Representations, 2020.

\bibitem[Hafner et~al.(2023)]{hafner2023mastering}
Danijar Hafner, Jurgis Pasukonis, Jimmy Ba, and Timothy Lillicrap. Mastering diverse domains through world models. arXiv preprint arXiv:2301.04104, 2023.

\bibitem[Hansen et~al.(2024)]{hansen2024tdmpc2}
Nicklas Hansen, Hao Su, and Xiaolong Wang. TD-MPC2: Scalable, robust world models for continuous control. In International Conference on Learning Representations, 2024.

\bibitem[Harshman(1970)]{harshman1970foundations}
Richard A. Harshman. Foundations of the PARAFAC procedure: Models and conditions for an ``explanatory'' multi-modal factor analysis. UCLA Working Papers in Phonetics 16, University of California, Los Angeles, 1970.

\bibitem[Janner et~al.(2019)]{janner2019trust}
Michael Janner, Justin Fu, Marvin Zhang, and Sergey Levine. When to trust your model: Model-based policy optimization. In Advances in Neural Information Processing Systems, volume 32, 2019.

\bibitem[Jin et~al.(2020)]{jin2020provably}
Chi Jin, Zhuoran Yang, Zhaoran Wang, and Michael I. Jordan. Provably efficient reinforcement learning with linear function approximation. In Proceedings of the Conference on Learning Theory, 2020.

\bibitem[Kolda and Bader(2009)]{kolda2009tensor}
Tamara G. Kolda and Brett W. Bader. Tensor decompositions and applications. SIAM Review, 51(3): 455--500, 2009. doi: 10.1137/07070111X.

\bibitem[Laskin et~al.(2020)]{laskin2020curl}
Michael Laskin, Aravind Srinivas, and Pieter Abbeel. CURL: Contrastive unsupervised representations for reinforcement learning. In Proceedings of the 37th International Conference on Machine Learning, 2020.

\bibitem[Lin et~al.(2024)]{lin2024because}
Haohong Lin, Wenhao Ding, Jian Chen, Laixi Shi, Jiacheng Zhu, Bo Li, and Ding Zhao. BECAUSE: Bilinear causal representation for generalizable offline model-based reinforcement learning. In Advances in Neural Information Processing Systems, 2024.

\bibitem[Lobel et~al.(2023)]{lobel2023qfunctionals}
Samuel Lobel, Sreehari Rammohan, Bowen He, Shangqun Yu, and George Konidaris. Q-functionals for value-based continuous control. In Proceedings of the AAAI Conference on Artificial Intelligence, volume 37, pages 8932--8939, 2023.

\bibitem[Ma and Collins(2018)]{ma2018noise}
Zhuang Ma and Michael Collins. Noise contrastive estimation and negative sampling for conditional models: Consistency and statistical efficiency. arXiv preprint arXiv:1809.01812, 2018.

\bibitem[Mahadevan(2005)]{mahadevan2005proto}
Sridhar Mahadevan. Proto-value functions: Developmental reinforcement learning. In Proceedings of the 22nd International Conference on Machine Learning, pages 553--560, 2005.

\bibitem[Mahajan et~al.(2021)]{mahajan2021tesseract}
Anuj Mahajan, Mikayel Samvelyan, Lei Mao, Viktor Makoviychuk, Animesh Garg, Jean Kossaifi, Shimon Whiteson, Yuke Zhu, and Animashree Anandkumar. Tesseract: Tensorised actors for multi-agent reinforcement learning. In Proceedings of the International Conference on Machine Learning, volume 139, pages 7301--7312. PMLR, 2021.

\bibitem[McDiarmid(1989)]{mcdiarmid1989method}
Colin McDiarmid. On the method of bounded differences. In J. Siemons, editor, Surveys in Combinatorics, 1989, volume 141 of London Mathematical Society Lecture Note Series, pages 148--188. Cambridge University Press, Cambridge, 1989.

\bibitem[Nabati et~al.(2026)]{tang2024spectral}
Ofir Nabati, Bo Dai, Shie Mannor, and Guy Tennenholtz. Spectral Bellman method: Unifying representation and exploration in RL. In International Conference on Learning Representations, 2026.

\bibitem[Ni et~al.(2023)]{ni2023learning}
Chengzhuo Ni, Yaqi Duan, Munther Dahleh, Mengdi Wang, and Anru R. Zhang. Learning good state and action representations for Markov decision process via tensor decomposition. Journal of Machine Learning Research, 24, 2023.

\bibitem[Rahimi and Recht(2007)]{rahimi2007random}
Ali Rahimi and Benjamin Recht. Random features for large-scale kernel machines. In Advances in Neural Information Processing Systems, 2007.

\bibitem[Ren et~al.(2023a)]{ren2023latent}
Tongzheng Ren, Chenjun Xiao, Tianjun Zhang, Na Li, Zhaoran Wang, Sujay Sanghavi, Dale Schuurmans, and Bo Dai. Latent variable representation for reinforcement learning. In International Conference on Learning Representations, 2023a.

\bibitem[Ren et~al.(2023b)]{ren2022spectral}
Tongzheng Ren, Tianjun Zhang, Lisa Lee, Joseph E. Gonzalez, Dale Schuurmans, and Bo Dai. Spectral decomposition representation for reinforcement learning. In International Conference on Learning Representations, 2023b.

\bibitem[Rozada et~al.(2024)]{rozada2024tensor}
Sergio Rozada, Santiago Paternain, and Antonio G. Marques. Tensor and matrix low-rank value-function approximation in reinforcement learning. IEEE Transactions on Signal Processing, 72: 1634--1649, 2024.

\bibitem[Rozada et~al.(2025)]{rozada2025tensor}
Sergio Rozada, Hoi-To Wai, and Antonio G. Marques. Multilinear tensor low-rank approximation for policy-gradient methods in reinforcement learning. IEEE Transactions on Signal Processing, 73: 4906--4920, 2025.

\bibitem[Schwarzer et~al.(2021)]{schwarzer2021data}
Max Schwarzer, Ankesh Anand, Rishab Goel, R. Devon Hjelm, Aaron Courville, and Philip Bachman. Data-efficient reinforcement learning with self-predictive representations. In International Conference on Learning Representations, 2021.

\bibitem[Shribak et~al.(2024)]{shribak2024diffusion}
Dmitry Shribak, Chen-Xiao Gao, Yitong Li, Chenjun Xiao, and Bo Dai. Diffusion spectral representation for reinforcement learning. In Advances in Neural Information Processing Systems, 2024.

\bibitem[Song et~al.(2021)]{song2021scorebased}
Yang Song, Jascha Sohl-Dickstein, Diederik P. Kingma, Abhishek Kumar, Stefano Ermon, and Ben Poole. Score-based generative modeling through stochastic differential equations. In International Conference on Learning Representations, 2021.

\bibitem[Stooke et~al.(2021)]{stooke2021decoupling}
Adam Stooke, Kimin Lee, Pieter Abbeel, and Michael Laskin. Decoupling representation learning from reinforcement learning. In Proceedings of the 38th International Conference on Machine Learning, 2021.

\bibitem[Tassa et~al.(2018)]{tassa2018dmcontrol}
Yuval Tassa, Yotam Doron, Alistair Muldal, Tom Erez, Yazhe Li, Diego de Las Casas, David Budden, Abbas Abdolmaleki, Josh Merel, Andrew Lefrancq, Timothy Lillicrap, and Martin Riedmiller. DeepMind Control Suite. arXiv preprint arXiv:1801.00690, 2018.

\bibitem[Touati and Ollivier(2021)]{touati2021learning}
Ahmed Touati and Yann Ollivier. Learning one representation to optimize all rewards. In Advances in Neural Information Processing Systems, volume 34, 2021.

\bibitem[Tucker(1966)]{tucker1966some}
Ledyard R. Tucker. Some mathematical notes on three-mode factor analysis. Psychometrika, 31(3): 279--311, 1966. doi: 10.1007/BF02289464.

\bibitem[Uehara et~al.(2022)]{uehara2022representation}
Masatoshi Uehara, Xuezhou Zhang, and Wen Sun. Representation learning for online and offline RL in low-rank MDPs. In International Conference on Learning Representations, 2022.

\bibitem[van~den Oord et~al.(2018)]{oord2018representation}
Aaron van den Oord, Yazhe Li, and Oriol Vinyals. Representation learning with contrastive predictive coding. arXiv preprint arXiv:1807.03748, 2018.

\bibitem[Wainwright(2019)]{wainwright2019high}
Martin J. Wainwright. High-Dimensional Statistics: A Non-Asymptotic Viewpoint. Cambridge University Press, 2019.

\bibitem[Wu et~al.(2019)]{wu2019laplacian}
Yifan Wu, George Tucker, and Ofir Nachum. The Laplacian in RL: Learning representations with efficient approximations. In International Conference on Learning Representations, 2019.

\bibitem[Yang and Wang(2020)]{yang2020reinforcement}
Lin Yang and Mengdi Wang. Reinforcement learning in feature space: Matrix bandit, kernels, and regret bound. In Proceedings of the International Conference on Machine Learning, 2020.

\bibitem[Zhang et~al.(2021)]{zhang2021learning}
Amy Zhang, Rowan McAllister, Roberto Calandra, Yarin Gal, and Sergey Levine. Learning invariant representations for reinforcement learning without reconstruction. In International Conference on Learning Representations, 2021.

\bibitem[Zhang et~al.(2022)]{zhang2022making}
Tianjun Zhang, Tongzheng Ren, Mengjiao Yang, Joseph E. Gonzalez, Dale Schuurmans, and Bo Dai. Making linear MDPs practical via contrastive representation learning. In Proceedings of the International Conference on Machine Learning, volume 162 of Proceedings of Machine Learning Research, pages 26447--26466. PMLR, 2022.

\bibitem[Zheng et~al.(2023)]{zheng2023taco}
Ruijie Zheng, Xiyao Wang, Yanchao Sun, Shuang Ma, Jieyu Zhao, Huazhe Xu, Hal Daumé III, and Furong Huang. TACO: Temporal latent action-driven contrastive loss for visual reinforcement learning. In Advances in Neural Information Processing Systems, 2023.

\end{thebibliography}
}{%

}
\newpage
\appendix
\section{Limitations}
\label{app:limitations}
The CP form $\phi_s(s)\odot\phi_a(a)$ treats the state and the action as two atomic axes. Environments with strongly coupled joints and tight state--action interdependence, such as Ant-style locomotion, fall outside the regime this two-way split captures cleanly, since the encoder must absorb all sub-axis interactions internally: on Hopper-Hop (contact switching) and Acrobot-Swingup (strong underactuated coupling), FaStR returns are below those of the joint encoder (Appendix~\ref{app:complementary-tasks}). A hierarchical factorization over actuator or body-part groups, e.g.\ $\phi_s(s)=\phi_{\text{torso}}\odot\phi_{\text{limbs}}$ or a tree-shaped CP, fits inside the same framework and is a clean next step. Second, diagonal $M$ corresponds to CP rank one along the next-state axis. A low-rank Tucker core $G\in\R^{d\times d}$ recovers more interaction terms but enlarges the effective NCE dimension from $d$ to $d^2$, reintroducing the constraint that motivated the diagonal choice. Mapping the rank--capacity trade-off, e.g.\ via low-rank rather than diagonal $M$, would close the gap between the CP and Tucker ends of this spectrum. Third, all experiments are conducted in the DMControl proprioceptive suite. Pixel-input variants and physical-robot deployment are not evaluated; the latter introduces sensor noise, latency, and actuator drift that the actuator-shift experiments only partially approximate. The evidence therefore supports the representation-learning advantages of FaStR in controlled simulation and does not replace validation on real robotic platforms; simulation experiments on legged robots such as Unitree Go2 or G1 are the next step. Finally, a limitation of the current evaluation is that, due to computational constraints, we did not evaluate the proposed method on every official DMControl task. Future work will extend the evaluation to the remaining tasks.
\section{Implementation Details}
\label{app:impl}

This appendix specifies the full training and evaluation protocol used in Section~\ref{sec:experiments}, together with the architectural and optimization hyperparameters of FaStR and the six baselines reported in Table~\ref{tab:full-results}: CTRL-SR~\citep{gao2025spectral}, Diff-SR~\citep{shribak2024diffusion}, SAC~\citep{haarnoja2018soft} in two width-depth variants (3-layer and 2-layer), TD7~\citep{fujimoto2023td7}, and TD-MPC2~\citep{hansen2024tdmpc2}. Hyperparameters that are common across methods (environment configuration, replay, evaluation cadence, total interaction budget) are listed once in Section~\ref{app:impl-protocol}; method-specific settings are listed in Sections~\ref{app:impl-fastr}--\ref{app:impl-tdmpc2} and summarized side-by-side in Table~\ref{tab:hyperparams}. We follow the convention of reporting every hyperparameter that was set in our runs, including those left at the original authors' defaults, so that the protocol is self-contained.

\subsection{Shared Training Protocol}
\label{app:impl-protocol}

All experiments use the DeepMind Control Suite~\citep{tassa2018dmcontrol} with proprioceptive observations, action repeat (frame skip) of 2, and episode length 1000 environment steps. We train each agent for $1{\times}10^6$ environment frames (i.e.\ $5{\times}10^5$ agent steps after action repeat) per task. Each agent maintains a single uniform-sampling replay buffer of capacity $1{\times}10^6$ transitions, with the first $10{,}000$ frames collected by a uniform random policy as a warm start before any gradient update. Evaluation is performed every $10{,}000$ environment frames by rolling out the deterministic actor (for TD3-based methods: TD3 mean action; for SAC: tanh-Gaussian mean) for $10$ episodes; we report the mean undiscounted return over the last 10 evaluation checkpoints, averaged over $8$ random seeds,
 with shaded bands or $\pm1$ standard deviations across seeds. Observation normalization (running mean and standard deviation) is enabled for the spectral methods (FaStR and CTRL-SR), following~\citet{gao2025spectral}, and disabled for the SAC and TD7 baselines, matching their original implementations. All methods use the discount factor $\gamma{=}0.99$, the same set of random seeds,
 and a single update per environment step (update-to-data ratio 1).

\subsection{FaStR (Ours)}
\label{app:impl-fastr}

\paragraph{Architecture.} All encoders, the noise schedule, and design defaults that are not explicitly varied below follow~\citet{gao2025spectral}. Encoders are ResidualMLPs with LayerNorm, ELU activations, and Mish activations inside the contrastive head. The state factor $\phis$ maps $\dim(\Scal)\to 512\to 512\to d$; the action factor $\phia$ maps $\dim(\Acal)\to 512\to 512\to d$; the kernel encoder $m$ maps $\dim(\Scal)+128\to 512\to 512\to d$, where the additional 128 dimensions encode the perturbation level via a sinusoidal time embedding. We set $d{=}512$ throughout. The factored feature $\psi(s,a)=\phis(s)\odot\phia(a)$ is fed to a twin Q critic consisting of a random Fourier features (RFF)~\citep{rahimi2007random} projection of width $1024$ followed by a 2-layer MLP head $[512]$ with LayerNorm and ELU; this critic is architecturally identical to the CTRL-SR critic so that the only varying component is the encoder. The actor is a deterministic policy network with hidden sizes $[512,512,512]$, ELU activations, and a final $\tanh$ that maps to $[-1,1]^{\dim(\Acal)}$.

\paragraph{Optimization.} All four networks (state factor, action factor, kernel encoder, critic) are trained with Adam (default $(\beta_1,\beta_2)=(0.9,0.999)$, no weight decay). Critic and actor learning rates are $3{\times}10^{-4}$; the representation learning rate (jointly applied to $\phis,\phia,m$) is $1{\times}10^{-4}$. Polyak averaging (i.e., $\bar\theta\leftarrow(1-\tau)\bar\theta+\tau\theta$) with $\tau{=}0.005$ is used for both the critic and the representation targets; actor and critic updates are performed every environment step (\texttt{actor\_update\_freq}=1). The mini-batch size is $N{=}512$, matching the negative-sample count $K{=}N{=}512$ implicit in the in-batch RP-NCE objective. Exploration uses Gaussian action noise $\Ncal(0,\sigma{=}0.2)$ clipped to $[-1,1]^{\dim(\Acal)}$; target action smoothing uses $\Ncal(0,\tilde\sigma{=}0.2)$ clipped to $[-0.3,0.3]$. Gradients of the critic loss are not back-propagated through the representation (\texttt{back\_critic\_grad}=false). One training step is summarized in Algorithm~\ref{alg:factored-sr}.

\paragraph{Representation objective.} We use the ranking-based perturbed NCE loss~\eqref{eq:nce} to train the trilinear transition-ratio score, with $M{=}25$ noise levels on a variance-preserving (VP) schedule. The ranking variant is enabled (\texttt{ranking}=true), and the auxiliary bilinear-MDP regression term of \citet{gao2025spectral} is kept at its default weight $\beta{=}1.0$. The combined training loss is $\ell_{\mathrm{FaStR}} + \lambda_r\,\ell_{\mathrm{reward}} + \lambda_Q\,\ell_{\mathrm{critic}}$ with $\lambda_{\mathrm{ctrl}}{=}\lambda_r{=}\lambda_Q{=}1.0$; the reward head is a 2-layer MLP $[512]$ on top of $\psi$, identical in structure to the CTRL-SR reward head.

\begin{algorithm}[h]
\caption{FaStR: one training step.  $\sg(\cdot)$: stop-gradient.}\label{alg:factored-sr}
\begin{algorithmic}[1]
\REQUIRE Encoders $\phis$, $\phia$, $m$; critic $Q_\theta$; actor $\pi_\xi$; target copies $\bar\phis$, $\bar\phia$, $\bar m$, $\bar Q$
\STATE Sample mini-batch $\{(s_i, a_i, r_i, s'_i)\}_{i=1}^N$; draw $K{-}1$ negatives per transition
\STATE \textbf{Representation:} $\psi_i \leftarrow \phis(s_i) \odot \phia(a_i)$\; for all $i$
\STATE \quad Update $(\phis, \phia, m)$ by minimizing $\ell_{\text{FaStR}}$ (Eq.~\ref{eq:nce})
\STATE \textbf{Critic:} $\bar\psi'_i \leftarrow \sg(\bar\phis(s'_i)) \odot \sg(\bar\phia(\pi_\xi(s'_i)))$;\; $y_i \leftarrow r_i + \gamma\, \bar Q(\bar\psi'_i)$
\STATE \quad Update $\theta$ by minimizing $\frac{1}{N}\sum_i (Q_\theta(\psi_i) - y_i)^2$
\STATE \textbf{Actor:} Update $\xi$ by maximizing $\frac{1}{N}\sum_i Q_\theta\!\bigl(\sg(\phis(s_i)) \odot \sg(\phia(\pi_\xi(s_i)))\bigr)$
\STATE \textbf{Targets:} $(\bar\phis, \bar\phia, \bar m, \bar Q) \leftarrow \tau\,(\phis, \phia, m, Q_\theta) + (1{-}\tau)\,(\bar\phis, \bar\phia, \bar m, \bar Q)$
\end{algorithmic}
\end{algorithm}

\subsection{CTRL-SR}
\label{app:impl-ctrlsr}

CTRL-SR~\citep{gao2025spectral} serves as the most controlled baseline: the only architectural difference from FaStR is that the transition-ratio logit uses a single monolithic encoder $\varphi:\Scal\times\Acal\to\R^d$ taking the concatenation $[s;a]$ as input, rather than a CP trilinear score built from separate state and action factors. The encoder $\varphi$ is a 2-layer ResidualMLP $[\dim(\Scal)+\dim(\Acal)] \to 512 \to 512 \to d$ with $d{=}512$, LayerNorm, and ELU activations; the kernel encoder $m$, RFF-MLP critic, deterministic actor $[512,512,512]$, reward head, optimizer (Adam), learning rates ($3{\times}10^{-4}$ for actor/critic and $1{\times}10^{-4}$ for representation), Polyak factor $\tau{=}0.005$, target/exploration action noise, batch size $N{=}512$, RP-NCE noise levels $M{=}25$, ranking objective, and update cadence are all identical to those of FaStR. We use the authors' open-source implementation without modification.

\subsection{Diff-SR}
\label{app:impl-diffsr}

Diff-SR~\citep{shribak2024diffusion} is the second spectral RL baseline. It uses the same TD3 actor and critic backbone as CTRL-SR and FaStR, but replaces the ranking-based perturbed NCE objective with a denoising-diffusion next-state model: the joint encoder $\phi:\Scal\times\Acal\to\R^d$ and the kernel encoder $\mu:\Scal\times[1,T]\to\R^{d\times\dim(\Scal)}$ are trained by score matching against the variance-preserving (VP) noise schedule of~\citet{shribak2024diffusion}, with the predicted noise factorized as $\hat\epsilon(s,a,s'_t,t)=\phi(s,a)^\top\mu(s'_t,t)$. The state and action inputs to $\phi$ are first projected by 2-layer Mish MLPs of widths $[256,128]$ to a shared embedding dimension of $128$ each, concatenated to a $256$-dim vector, and passed to a 3-layer ResidualMLP $256\to512\to512\to512\to d$ with $d{=}512$. The kernel encoder $\mu$ takes $(s'_t,t)$ as input, where $t$ is encoded by a sinusoidal positional embedding of width $128$ followed by a 2-layer Mish MLP, and outputs a $d\times\dim(\Scal)$ matrix via a 3-layer ResidualMLP $[\dim(\Scal)+128]\to 512\to 512\to 512\to d\cdot\dim(\Scal)$. The diffusion uses $T{=}50$ sample steps with the noisy state clamped to $[-10,10]$. The critic is the same RFF-MLP twin Q as CTRL-SR (RFF projection of width $1024$ followed by a 1-hidden-layer $[512]$ head); the deterministic actor is a $[512,512,512]$ MLP with LayerNorm and a final $\tanh$; the reward head is an RFF-MLP with the same shape as the critic head.

All four networks are trained with Adam. Critic and actor learning rates are $3{\times}10^{-4}$; the representation learning rate (jointly applied to $\phi,\mu$, and the reward head) is $1{\times}10^{-4}$. Polyak averaging uses $\tau{=}0.005$, and actor and critic updates are performed every environment step (\texttt{actor\_update\_freq}=1, \texttt{feature\_update\_ratio}=1). Mini-batch size is $N{=}512$, matching the spectral baselines. Exploration uses Gaussian noise $\Ncal(0,\sigma{=}0.2)$ clipped to $[-1,1]^{\dim(\Acal)}$; target action smoothing uses $\Ncal(0,\tilde\sigma{=}0.2)$ clipped to $[-0.3,0.3]$. The combined training loss is $\lambda_{\mathrm{diff}}\,\ell_{\mathrm{diffusion}} + \lambda_r\,\ell_{\mathrm{reward}} + \lambda_Q\,\ell_{\mathrm{critic}}$ with $\lambda_{\mathrm{diff}}{=}1.0$, $\lambda_r{=}0.1$, and $\lambda_Q{=}1.0$; gradients of the critic loss are not back-propagated through the representation (\texttt{back\_critic\_grad}=false). Observation normalization is enabled, matching CTRL-SR and FaStR. We use the open-source implementation accompanying~\citet{shribak2024diffusion} without modification.

\subsection{SAC}
\label{app:impl-sac}

SAC~\citep{haarnoja2018soft} uses a stochastic tanh-Gaussian policy and a twin Q critic with ELU activations and LayerNorm. The critic takes the raw $(s,a)$ concatenation as input; no representation learning is performed. We report two width-depth variants. SAC(3-layer) uses hidden sizes $[512,512,512]$ for both the actor and the critic, aligning its hidden width with the spectral methods and removing hidden width as an avoidable source of variation. SAC(2-layer) uses hidden sizes $[1024,1024]$ for both.
 All other settings are shared between the two variants. Following~\citet{haarnoja2018soft}, the entropy coefficient $\alpha$ is learned online (automatic entropy tuning) toward a target entropy of $-\dim(\Acal)$, optimized via Adam and learning rate $3{\times}10^{-4}$ from an initial value $\alpha_0{=}0.2$. Actor, critic, and entropy coefficient are all trained with Adam at learning rate $3{\times}10^{-4}$; soft target updates use Polyak factor $\tau{=}0.005$ at every environment step, and the actor is updated every step (\texttt{actor\_update\_freq}=1). Mini-batch size is $512$ for parity with the spectral methods. There is no exploration noise injected on top of the policy stochasticity, no observation normalization, and no gradient clipping.

\subsection{TD7}
\label{app:impl-td7}

TD7~\citep{fujimoto2023td7} extends TD3 with a learned state-action encoder ($zs,zsa$) trained by latent dynamics regression, Loss-Adjusted Prioritized (LAP) replay with per-transition priorities, behavior-cloning regularization at the policy update, and policy checkpointing. We use the official implementation released by the authors with the original DM Control Suite hyperparameters; the only configuration we set is the task identifier and the seed. For completeness: encoder, actor, and critic each use 2-hidden-layer MLPs of width $256$ with ELU activations (ReLU for the actor) and the AvgL1Norm normalization of~\citet{fujimoto2023td7} (per-layer activation rescaling); the encoded state-action dimension is $z_{\mathrm{dim}}{=}256$. Optimizer is Adam with learning rate $3{\times}10^{-4}$ for all three networks. Replay uses LAP with $\alpha{=}0.4$, $\min$-priority $1$, and Huber critic loss; the discount is $\gamma{=}0.99$ and the target network is updated by a hard copy every $250$ updates (rather than Polyak averaging). Exploration uses Gaussian noise $\Ncal(0,0.1)$ clipped to $[-1,1]$; target policy smoothing uses $\Ncal(0,0.2)$ clipped to $[-0.5,0.5]$; the actor is updated every $2$ critic updates with behavior-cloning weight $\lambda_{\mathrm{BC}}{=}0.1$. Policy checkpointing uses up to $20$ episodes per checkpoint and starts after $7.5{\times}10^5$ agent steps with weight-reset coefficient $0.9$. Mini-batch size is $256$ (the value used in the original TD7 paper); we did not tune this for parity with the spectral methods. Random exploration runs for the first $25{,}000$ agent steps before any training begins.

\subsection{TD-MPC2}
\label{app:impl-tdmpc2}

TD-MPC2~\citep{hansen2024tdmpc2} is the model-based baseline. We use the official default 5M-parameter single-task configuration of \citet{hansen2024tdmpc2} without modification, integrated into the shared training and evaluation pipeline of Section~\ref{app:impl-protocol} and trained for the same $1{\times}10^6$ environment-interaction budget, with results reported over $8$ seeds. DreamerV3~\citep{hafner2023mastering} under its standard configuration requires substantially more compute than our available budget, and running it at a much smaller budget would not constitute a fair comparison, so it is not included.

\subsection{Hyperparameter Summary}
\label{app:impl-summary}

Table~\ref{tab:hyperparams} consolidates the values reported above. Dashes (\textemdash) indicate that a hyperparameter does not apply to a given method (e.g., a representation learning rate for SAC). Where a method has additional method-specific hyperparameters not listed in the table (LAP $\alpha$ for TD7, BC weight, target entropy for SAC, RP-NCE noise schedule for the spectral methods), we have specified them in the corresponding paragraph above and use the original authors' default values.

\begin{table}[ht]
\centering
\caption{Hyperparameters for FaStR and the SAC, TD7, and spectral baselines used in Table~\ref{tab:full-results}; TD-MPC2 uses its official default single-task configuration (Section~\ref{app:impl-tdmpc2}). Values shared across all methods (discount $\gamma{=}0.99$, $1{\times}10^6$ environment frames, replay capacity $1{\times}10^6$, action repeat $2$, $8$ seeds, $10$ evaluation episodes every $10{,}000$ frames) are described in Section~\ref{app:impl-protocol}.}
\label{tab:hyperparams}
\footnotesize
\setlength{\tabcolsep}{4pt}
\resizebox{\textwidth}{!}{%
\begin{tabular}{@{}l cccccc@{}}
\toprule
Parameter & FaStR (Ours) & CTRL-SR & Diff-SR & SAC(3-layer) & SAC(2-layer) & TD7 \\
\midrule
Optimizer                          & Adam          & Adam          & Adam          & Adam          & Adam          & Adam \\
Actor learning rate                & $3{\times}10^{-4}$ & $3{\times}10^{-4}$ & $3{\times}10^{-4}$ & $3{\times}10^{-4}$ & $3{\times}10^{-4}$ & $3{\times}10^{-4}$ \\
Critic learning rate               & $3{\times}10^{-4}$ & $3{\times}10^{-4}$ & $3{\times}10^{-4}$ & $3{\times}10^{-4}$ & $3{\times}10^{-4}$ & $3{\times}10^{-4}$ \\
Representation/encoder LR          & $1{\times}10^{-4}$ & $1{\times}10^{-4}$ & $1{\times}10^{-4}$ & \textemdash      & \textemdash      & $3{\times}10^{-4}$ \\
Mini-batch size                    & 512           & 512           & 512           & 512           & 512           & 256 \\
Target update                      & Polyak $\tau{=}0.005$ & Polyak $\tau{=}0.005$ & Polyak $\tau{=}0.005$ & Polyak $\tau{=}0.005$ & Polyak $\tau{=}0.005$ & Hard, every 250 updates \\
Actor update frequency             & 1             & 1             & 1             & 1             & 1             & 2 \\
Actor hidden sizes                 & [512,512,512] & [512,512,512] & [512,512,512] & [512,512,512] & [1024,1024] & [256,256] \\
Critic hidden sizes                & [512] (after RFF 1024) & [512] (after RFF 1024) & [512] (after RFF 1024) & [512,512,512] & [1024,1024] & [256,256] \\
Activation                         & ELU/Mish + LN & ELU/Mish + LN & Mish + LN     & ELU + LN      & ELU + LN      & ELU + AvgL1Norm \\
Number of critics                  & 2 (twin Q)    & 2 (twin Q)    & 2 (twin Q)    & 2 (twin Q)    & 2 (twin Q)    & 2 (twin Q) \\
Exploration noise $\sigma$         & 0.2           & 0.2           & 0.2           & policy entropy & policy entropy & 0.1 \\
Target policy noise $\tilde\sigma$ & 0.2 (clip $\pm0.3$) & 0.2 (clip $\pm0.3$) & 0.2 (clip $\pm0.3$) & \textemdash & \textemdash & 0.2 (clip $\pm0.5$) \\
Feature dim $d$                    & 512           & 512           & 512           & \textemdash      & \textemdash      & 256 ($z_{\mathrm{dim}}$) \\
Encoder hidden sizes               & [512,512]     & [512,512]     & [512,512,512] & \textemdash      & \textemdash      & [256,256] \\
RFF projection width               & 1024          & 1024          & 1024          & \textemdash      & \textemdash      & \textemdash \\
Representation objective           & RP-NCE (ranking) & RP-NCE (ranking) & VP diffusion ($T{=}50$) & \textemdash & \textemdash & latent dynamics regression \\
RP-NCE noise levels $M$            & 25            & 25            & \textemdash    & \textemdash      & \textemdash      & \textemdash \\
Negative samples $K$               & 512 (in-batch) & 512 (in-batch) & \textemdash   & \textemdash      & \textemdash      & \textemdash \\
Replay sampling                    & uniform       & uniform       & uniform       & uniform       & uniform       & LAP ($\alpha{=}0.4$) \\
Observation normalization          & running stats & running stats & running stats & none          & none          & none \\
Random/warmup frames               & 10{,}000      & 10{,}000      & 10{,}000      & 10{,}000      & 10{,}000      & 50{,}000 (25k agent steps) \\
\bottomrule
\end{tabular}%
}
\end{table}

\section{Proof Details}
\subsection{Notation, Conventions, and Standing Assumptions}
\label{app:notation}
We work with a discounted MDP $(\Scal,\Acal,P,r,\gamma,\rho_0)$ on a continuous state space $\Scal$ and continuous action space $\Acal$, with $\gamma \in [0,1)$ and reward bounded by $R_{\max}$. We assume the existence of a fixed reference \emph{probability measure} $\nu$ on $\Scal$ that dominates every transition kernel, so that $P(\cdot\mid s,a)\ll \nu$ and admits a density $P(s'\mid s,a)$ with respect to $\nu$. All integrals over $\Scal$ are then taken with respect to $\nu$, i.e.\ $\int_\Scal f(s')\,ds' \equiv \int_\Scal f(s')\,d\nu(s')$. This is the standard setting for $L^2(\nu)$-based representation-learning analyses and fixes the measure with respect to which all densities and risks are evaluated.

\paragraph{Symbols.}
\begin{itemize}[leftmargin=2em,itemsep=2pt]
\item $\phis:\Scal\to\R^d$, $\phia:\Acal\to\R^d$, $m:\Scal\to\R^d$: state encoder, action encoder, kernel encoder.
\item $\psi(s,a) := \phis(s)\odot\phia(a) \in \R^d$: \emph{factored} (Hadamard) feature.
\item $\varphi(s,a)\in\R^d$, $\mu:\Scal\to\R^d$: \emph{joint} (monolithic) encoder and kernel encoder.
\item $\Fcal_s,\Fcal_a,\Fcal_m,\Fcal_{sa}$: hypothesis classes for $\phis,\phia,m,\varphi$.
\item $\Gcal_{\mathrm{fac}} := \{(s,a,s')\mapsto (\phis(s)\odot\phia(a))^\top m(s'):\phis\in\Fcal_s,\phia\in\Fcal_a, m\in\Fcal_m\}$.
\item $\Gcal_{\mathrm{joint}} := \{(s,a,s')\mapsto \varphi(s,a)^\top\mu(s'):\varphi\in\Fcal_{sa},\mu\in\Fcal_m\}$.
\end{itemize}

\paragraph{Boundedness assumptions.}
We assume there exist constants $B_\phi,B_\varphi,B_{m,\infty},B_{m,1},B_g \ge 1$ such that, uniformly over their hypothesis classes,
\begin{align*}
\sup_{\phis\in\Fcal_s}\sup_{s\in\Scal}\|\phis(s)\|_2,\,
\sup_{\phia\in\Fcal_a}\sup_{a\in\Acal}\|\phia(a)\|_2 &\le B_\phi,
&
\sup_{\varphi\in\Fcal_{sa}}\sup_{(s,a)}\|\varphi(s,a)\|_2 &\le B_\varphi, \\
\sup_{h\in\Fcal_m}\sup_{s'\in\Scal}\|h(s')\|_2 &\le B_{m,\infty},
&
\sup_{h\in\Fcal_m}\int_\Scal \|h(s')\|_2\,d\nu(s') &\le B_{m,1},
\end{align*}
and $|g(s,a,s')|\le B_g$ for every $g\in\Gcal_{\mathrm{fac}}\cup\Gcal_{\mathrm{joint}}$. Thus both the factored kernel encoder $m$ and the joint kernel encoder $\mu$ obey the same pointwise and $L^1(\nu)$ bounds. For a single continuous encoder on compact $\Scal$ these bounds are automatic; for a hypothesis class we impose them uniformly. We separate the pointwise bound $B_{m,\infty}$ (used for covering numbers and sup-norm Lipschitz arguments) from the integral bound $B_{m,1}$ (used to verify that integrals against $V^\pi$ are finite); the two regimes appear in different proofs.

\paragraph{Rademacher complexity and covering numbers.}
For a fixed sample $S=(z_1,\dots,z_n)$ and a scalar-valued class $\Hcal$, define the empirical absolute Rademacher complexity
\[
\widehat\Rcal_S(\Hcal)
:=
\E_\sigma\!\left[
\sup_{h\in\Hcal}
\left|
\frac{1}{n}\sum_{i=1}^n \sigma_i\, h(z_i)
\right|
\right],
\qquad
\sigma_i\stackrel{\mathrm{iid}}{\sim}\mathrm{Unif}\{\pm1\}.
\]
The expected Rademacher complexity is
\[
\Rcal_n(\Hcal):=\E_S\!\left[\widehat\Rcal_S(\Hcal)\right].
\]
All finite-sample bounds use this expected version. Set-inclusion arguments hold pointwise for $\widehat\Rcal_S$ and therefore also after taking expectation over $S$.
For scalar-valued classes, $\Ncal(\Hcal,\epsilon)$ denotes the covering number under
$d_\infty(h,h')=\sup_z |h(z)-h'(z)|$. For vector-valued encoder classes such as $\Fcal_s,\Fcal_a,\Fcal_m,\Fcal_{sa}$, the covering metric is
\[
d_\infty(f,f')=\sup_x\|f(x)-f'(x)\|_2.
\]

\paragraph{Best-in-class CP error.} For non-realizable settings we write
\[
\epsilon_{\mathrm{CP}}^2 \;:=\; \inf_{g\in\Gcal_{\mathrm{fac}}}\; \E_{(s,a)\sim\rho}\!\left[\int_\Scal \bigl(P(s'\mid s,a) - g(s,a,s')\bigr)^2\,d\nu(s')\right].
\]
Under CP realizability (Assumption~\ref{assump:factored}), $\epsilon_{\mathrm{CP}}=0$.

\bigskip

\paragraph{Proof architecture.}
The results in this appendix support a single claim: the CP factorization of the transition kernel gives a representation whose statistical complexity decomposes additively in $(d_s,d_a)$, whereas the generic uniform-convergence bound for an unrestricted joint encoder scales with the full joint dimension $d_s+d_a$.
Appendix~\ref{app:proof-bilinear-Q} derives the Q-linearity stated in Section~\ref{sec:method-cp}: under CP factorization, every $Q^\pi$ is linear in the Hadamard feature, so the factored class is matched to the structure of the transition kernel rather than chosen as an architectural convenience.
Appendix~\ref{app:proof-complexity} proves Theorem~\ref{thm:complexity}, the main technical step: the Hadamard product converts the joint covering of $\Gcal_{\mathrm{fac}}$ into a sum of three independent component coverings, and the same set-inclusion argument gives a Rademacher ordering between the two classes. Appendix~\ref{app:proof-rademacher} states this Rademacher monotonicity on its own, since it follows from set inclusion without needing the covering bound.
Appendix~\ref{app:proof-finite-sample} converts the Rademacher bound into an excess-risk bound for the spectral $L_2$ objective with constants made explicit (Proposition~\ref{prop:finite-sample}). Appendix~\ref{app:proof-sample-sep} converts the additive vs.\ multiplicative entropy gap into the sufficient-sample-size gap implied by these bounds (Proposition~\ref{prop:sample-sep}). It ends with a same-$n$ comparison that holds without any Lipschitz instantiation: under the containment condition, both Rademacher terms in Proposition~\ref{prop:finite-sample} are no larger for the factored class, so the certified upper bound on representation error is no larger at every sample size (Remark~\ref{rmk:same-n}).

\subsection{Derivation of the Q-linearity under CP}
\label{app:proof-bilinear-Q}

The CP transition structure makes every $Q^\pi$ linear in the Hadamard feature $\psi$.
This appendix gives the linear-MDP derivation specialized to the CP feature, together
with norm bounds on the value coefficients used downstream. We spell out the
interchange of summation and integration and the norm estimates in detail.

\paragraph{Setup.}
Suppose Assumption~\ref{assump:factored} holds. That is, the reward decomposes as
\[
r(s,a)=\psi(s,a)^\top \theta_r,
\]
where
\[
\psi(s,a)
=
\phi_s(s)\odot \phi_a(a)
\in \R^d,
\qquad
\psi_k(s,a)=\phi_{s,k}(s)\phi_{a,k}(a),
\]
and the transition density factorizes as
\begin{equation}\label{eq:cp-density}
P(s'\mid s,a)
=
\sum_{k=1}^d
\phi_{s,k}(s)\,\phi_{a,k}(a)\,m_k(s')
=
\sum_{k=1}^d
\psi_k(s,a)m_k(s').
\end{equation}
Here each $m_k$ is measurable and belongs to $L^1(\nu)$. We also use the
boundedness conventions of Appendix~\ref{app:notation}; in particular, writing
\[
m(s')=(m_1(s'),\ldots,m_d(s'))^\top,
\]
we have
\[
\int_\Scal \|m(s')\|_2\,d\nu(s')\le B_{m,1}.
\]
Consequently, for every coordinate $k$,
\[
\int_\Scal |m_k(s')|\,d\nu(s')
\le
\int_\Scal \|m(s')\|_2\,d\nu(s')
\le B_{m,1}.
\]

\paragraph{Claim.}
For every stationary policy $\pi$, there exists $w^\pi\in\R^d$ such that
\[
Q^\pi(s,a)=\psi(s,a)^\top w^\pi,
\]
where
\[
w^\pi_k=\theta_{r,k}+\gamma\theta^\pi_k,
\qquad
\theta^\pi_k
:=
\int_\Scal m_k(s')V^\pi(s')\,d\nu(s').
\]
Moreover,
\[
\|w^\pi\|_2
\le
\|\theta_r\|_2
+
\gamma\,\frac{B_{m,1}R_{\max}}{1-\gamma}.
\]

\begin{proof}
Fix an arbitrary stationary policy $\pi$. We prove the result for this fixed
$\pi$.

First, by the definition of the discounted value function,
\[
V^\pi(s)
=
\E_\pi\!\left[
\sum_{t=0}^{\infty}\gamma^t r(s_t,a_t)
\,\middle|\,s_0=s
\right].
\]
Since $|r(s,a)|\le R_{\max}$ for all $(s,a)$ and $0\le \gamma<1$, we have, for
every $s$,
\begin{align*}
|V^\pi(s)|
&=
\left|
\E_\pi\!\left[
\sum_{t=0}^{\infty}\gamma^t r(s_t,a_t)
\,\middle|\,s_0=s
\right]
\right|                                                     \\
&\le
\E_\pi\!\left[
\left|
\sum_{t=0}^{\infty}\gamma^t r(s_t,a_t)
\right|
\,\middle|\,s_0=s
\right]                                                     \\
&\le
\E_\pi\!\left[
\sum_{t=0}^{\infty}\gamma^t |r(s_t,a_t)|
\,\middle|\,s_0=s
\right]                                                     \\
&\le
\E_\pi\!\left[
\sum_{t=0}^{\infty}\gamma^t R_{\max}
\,\middle|\,s_0=s
\right]                                                     \\
&=
R_{\max}\sum_{t=0}^{\infty}\gamma^t
=
\frac{R_{\max}}{1-\gamma}.
\end{align*}
Therefore,
\begin{equation}\label{eq:value-uniform-bound}
\|V^\pi\|_\infty
\le
\frac{R_{\max}}{1-\gamma}.
\end{equation}

The Bellman equation for $Q^\pi$ is
\begin{equation}\label{eq:bellman}
Q^\pi(s,a)
=
r(s,a)
+
\gamma
\int_\Scal
P(s'\mid s,a)V^\pi(s')\,d\nu(s').
\end{equation}
We now expand the transition term in \eqref{eq:bellman}. Substituting the
CP density representation \eqref{eq:cp-density} gives
\begin{align}
\int_\Scal P(s'\mid s,a)V^\pi(s')\,d\nu(s')
&=
\int_\Scal
\left(
\sum_{k=1}^d
\phi_{s,k}(s)\phi_{a,k}(a)m_k(s')
\right)
V^\pi(s')\,d\nu(s')                                      \notag\\
&=
\int_\Scal
\left(
\sum_{k=1}^d
\psi_k(s,a)m_k(s')
\right)
V^\pi(s')\,d\nu(s').
\label{eq:expand-cp}
\end{align}

Before moving the finite sum outside the integral, we check that each term is
integrable. For a fixed coordinate $k$, using \eqref{eq:value-uniform-bound},
\begin{align*}
\int_\Scal
\left|
m_k(s')V^\pi(s')
\right|
\,d\nu(s')
&\le
\int_\Scal
|m_k(s')|\,|V^\pi(s')|
\,d\nu(s')                                                \\
&\le
\frac{R_{\max}}{1-\gamma}
\int_\Scal
|m_k(s')|
\,d\nu(s')                                                \\
&\le
\frac{R_{\max}}{1-\gamma}
B_{m,1}
<
\infty.
\end{align*}
Thus $m_kV^\pi$ is integrable for every $k$, and the quantity
\[
\theta^\pi_k
:=
\int_\Scal m_k(s')V^\pi(s')\,d\nu(s')
\]
is well-defined.

Since the sum over $k$ contains only finitely many terms, exchanging the sum and
the integral uses only the elementary linearity of the integral. Therefore,
\begin{align}
\int_\Scal
\left(
\sum_{k=1}^d
\psi_k(s,a)m_k(s')
\right)
V^\pi(s')\,d\nu(s')
&=
\int_\Scal
\sum_{k=1}^d
\psi_k(s,a)m_k(s')V^\pi(s')\,d\nu(s')                 \notag\\
&=
\sum_{k=1}^d
\int_\Scal
\psi_k(s,a)m_k(s')V^\pi(s')\,d\nu(s')                 \notag\\
&=
\sum_{k=1}^d
\psi_k(s,a)
\int_\Scal
m_k(s')V^\pi(s')\,d\nu(s')                            \notag\\
&=
\sum_{k=1}^d
\psi_k(s,a)\theta^\pi_k.
\label{eq:transition-linear}
\end{align}
In the third line above, $\psi_k(s,a)$ is taken outside the integral because,
for fixed $(s,a)$, it is a scalar that does not depend on $s'$.

Next, we expand the reward term. Since $r(s,a)=\psi(s,a)^\top\theta_r$, we have
\begin{equation}\label{eq:reward-linear}
r(s,a)
=
\sum_{k=1}^d
\psi_k(s,a)\theta_{r,k}
=
\sum_{k=1}^d
\phi_{s,k}(s)\phi_{a,k}(a)\theta_{r,k}.
\end{equation}

Combining \eqref{eq:bellman}, \eqref{eq:transition-linear}, and
\eqref{eq:reward-linear}, we obtain
\begin{align*}
Q^\pi(s,a)
&=
r(s,a)
+
\gamma
\int_\Scal P(s'\mid s,a)V^\pi(s')\,d\nu(s')                  \\
&=
\sum_{k=1}^d
\psi_k(s,a)\theta_{r,k}
+
\gamma
\sum_{k=1}^d
\psi_k(s,a)\theta^\pi_k                                      \\
&=
\sum_{k=1}^d
\psi_k(s,a)\theta_{r,k}
+
\sum_{k=1}^d
\psi_k(s,a)\gamma\theta^\pi_k                                \\
&=
\sum_{k=1}^d
\psi_k(s,a)
\left(
\theta_{r,k}+\gamma\theta^\pi_k
\right).
\end{align*}
Define $w^\pi\in\R^d$ coordinatewise by
\[
w^\pi_k
:=
\theta_{r,k}+\gamma\theta^\pi_k,
\qquad k=1,\ldots,d.
\]
Then the previous display becomes
\[
Q^\pi(s,a)
=
\sum_{k=1}^d
\psi_k(s,a)w^\pi_k
=
\psi(s,a)^\top w^\pi.
\]
This proves the claimed linear representation of $Q^\pi$.

It remains to prove the stated norm bound. By definition,
\[
\theta^\pi
=
\left(
\int_\Scal m_1(s')V^\pi(s')\,d\nu(s'),
\ldots,
\int_\Scal m_d(s')V^\pi(s')\,d\nu(s')
\right)^\top.
\]
Equivalently, in vector notation,
\[
\theta^\pi
=
\int_\Scal m(s')V^\pi(s')\,d\nu(s'),
\]
where the vector integral is understood coordinatewise.

We show that
\[
\|\theta^\pi\|_2
\le
\frac{B_{m,1}R_{\max}}{1-\gamma}.
\]
If $\theta^\pi=0$, this inequality is immediate. Otherwise, let
\[
u
=
\frac{\theta^\pi}{\|\theta^\pi\|_2}.
\]
Then $\|u\|_2=1$, and hence
\begin{align*}
\|\theta^\pi\|_2
&=
u^\top \theta^\pi                                             \\
&=
u^\top
\int_\Scal m(s')V^\pi(s')\,d\nu(s')                            \\
&=
\int_\Scal u^\top m(s')V^\pi(s')\,d\nu(s')                      \\
&\le
\int_\Scal
\left|
u^\top m(s')V^\pi(s')
\right|
\,d\nu(s')                                                     \\
&=
\int_\Scal
\left|
u^\top m(s')
\right|
\left|
V^\pi(s')
\right|
\,d\nu(s')                                                     \\
&\le
\int_\Scal
\|u\|_2\|m(s')\|_2
\left|
V^\pi(s')
\right|
\,d\nu(s')                                                     \\
&=
\int_\Scal
\|m(s')\|_2
\left|
V^\pi(s')
\right|
\,d\nu(s')                                                     \\
&\le
\frac{R_{\max}}{1-\gamma}
\int_\Scal
\|m(s')\|_2
\,d\nu(s')                                                     \\
&\le
\frac{B_{m,1}R_{\max}}{1-\gamma}.
\end{align*}
The inequality
\[
|u^\top m(s')|\le \|u\|_2\|m(s')\|_2
\]
is the Cauchy--Schwarz inequality. The key point is that this argument bounds
the Euclidean norm of the whole vector $\theta^\pi$ directly, and therefore
does not introduce an unnecessary $\sqrt d$ factor.

Finally, since
\[
w^\pi=\theta_r+\gamma\theta^\pi,
\]
the triangle inequality gives
\begin{align*}
\|w^\pi\|_2
&=
\|\theta_r+\gamma\theta^\pi\|_2                                  \\
&\le
\|\theta_r\|_2+\|\gamma\theta^\pi\|_2                              \\
&=
\|\theta_r\|_2+\gamma\|\theta^\pi\|_2                              \\
&\le
\|\theta_r\|_2
+
\gamma\,\frac{B_{m,1}R_{\max}}{1-\gamma}.
\end{align*}
This proves the claimed norm bound and completes the proof.
\end{proof}

\paragraph{Remark (Linear-MDP structure and feature bound)}\label{rmk:linear-mdp}
Defining signed measures $\mu_k(A):=\int_A m_k(s')\,d\nu(s')$, the CP transition density is equivalent to $P(\cdot\mid s,a)=\sum_{k=1}^d\psi_k(s,a)\mu_k(\cdot)$. Together with $r(s,a)=\psi(s,a)^\top\theta_r$, this matches the linear-MDP definition of \citet[Definition~2]{jin2020provably} with feature map $\psi$, used in Appendix~\ref{app:proof-sample-sep}. The same swap-of-sum-and-integral argument as the proof of Q-linearity above gives, for any bounded measurable $f:\Scal\to\R$, $\int_\Scal P(s'\mid s,a)f(s')\,d\nu(s')=\psi(s,a)^\top\theta_f$ with $(\theta_f)_k:=\int_\Scal m_k(s')f(s')\,d\nu(s')$ and $\|\theta_f\|_2\le B_{m,1}\|f\|_\infty$. Finally, the Hadamard-product inequality $\|u\odot v\|_2\le\|u\|_2\|v\|_2$ applied to $\psi=\phi_s\odot\phi_a$ gives the uniform feature bound $\|\psi(s,a)\|_2\le B_\phi^2$.

\subsection{Proof of Theorem~\ref{thm:complexity}: Additive Covering Number Decomposition}
\label{app:proof-complexity}

This is the main technical step. The Hadamard product structure converts the joint
covering of $\Gcal_{\mathrm{fac}}$ into a product of three independent coverings
on the state, action, and next-state components. After taking logarithms, this
product becomes an additive decomposition of metric entropies. Without using this
decomposition, a generic uniform-convergence treatment of the state-action pair
would cover the whole state-action feature class directly and would therefore give
the same $(d_s+d_a)$-dimensional entropy scaling as the joint class. The
Hadamard structure reduces the factored upper bound to the largest of the state
and action dimensions.

Throughout this appendix, covering numbers for the component classes are taken
with respect to the uniform sup-$\ell_2$ metrics
\[
d_s(\phis,\widetilde \phis)
:=
\sup_{s\in\Scal}
\|\phis(s)-\widetilde \phis(s)\|_2,
\]
\[
d_a(\phia,\widetilde \phia)
:=
\sup_{a\in\Acal}
\|\phia(a)-\widetilde \phia(a)\|_2,
\]
and
\[
d_m(m,\widetilde m)
:=
\sup_{s'\in\Scal}
\|m(s')-\widetilde m(s')\|_2.
\]
For scoring functions $g:\Scal\times\Acal\times\Scal\to\R$, the covering number
is taken with respect to the usual uniform norm
\[
\|g-\widetilde g\|_\infty
:=
\sup_{(s,a,s')\in\Scal\times\Acal\times\Scal}
|g(s,a,s')-\widetilde g(s,a,s')|.
\]

\begin{theorem*}[Restatement of Theorem~\ref{thm:complexity}]\mbox{}
\begin{enumerate}[label=\emph{(\roman*)}]
\item \emph{Additive decomposition.} For every $\epsilon>0$,
\begin{equation}\label{eq:covering-decomp}
\log\Ncal(\Gcal_{\mathrm{fac}},\epsilon)
\le
\log\Ncal\!\Bigl(\Fcal_s,\tfrac{\epsilon}{3B_{m,\infty}B_\phi}\Bigr)
+
\log\Ncal\!\Bigl(\Fcal_a,\tfrac{\epsilon}{3B_{m,\infty}B_\phi}\Bigr)
+
\log\Ncal\!\Bigl(\Fcal_m,\tfrac{\epsilon}{3B_\phi^2}\Bigr).
\end{equation}
\item \emph{Containment.} If
\[
\Fcal_{sa}\supseteq
\{\phis\odot\phia:\phis\in\Fcal_s,\phia\in\Fcal_a\},
\]
then $\Gcal_{\mathrm{fac}}\subseteq\Gcal_{\mathrm{joint}}$, hence
\[
\Ncal(\Gcal_{\mathrm{fac}},\epsilon)
\le
\Ncal(\Gcal_{\mathrm{joint}},\epsilon)
\]
for every $\epsilon>0$.
\end{enumerate}
\end{theorem*}

\begin{lemma}[Perturbation bound for factored CP scores]
\label{lem:cp-score-perturb}
Let
\[
g(s,a,s')=(\phis(s)\odot\phia(a))^\top m(s')
\]
and
\[
\widetilde g(s,a,s')
=
(\widetilde\phis(s)\odot\widetilde\phia(a))^\top\widetilde m(s').
\]
Assume
\[
\sup_s\|\phis(s)\|_2,\ \sup_s\|\widetilde\phis(s)\|_2\le B_\phi,
\qquad
\sup_a\|\phia(a)\|_2,\ \sup_a\|\widetilde\phia(a)\|_2\le B_\phi,
\]
and
\[
\sup_{s'}\|\widetilde m(s')\|_2\le B_{m,\infty}.
\]
Define
\[
\Delta_s:=\sup_s\|\phis(s)-\widetilde\phis(s)\|_2,
\qquad
\Delta_a:=\sup_a\|\phia(a)-\widetilde\phia(a)\|_2,
\]
and
\[
\Delta_m:=\sup_{s'}\|m(s')-\widetilde m(s')\|_2.
\]
Then
\begin{equation}\label{eq:sup-bound}
\|g-\widetilde g\|_\infty
\le
B_\phi^2\Delta_m
+
B_{m,\infty}B_\phi\Delta_s
+
B_{m,\infty}B_\phi\Delta_a.
\end{equation}
\end{lemma}
\begin{proof}
Write
\[
x(s,a):=\phis(s)\odot\phia(a),
\qquad
\widetilde x(s,a):=\widetilde\phis(s)\odot\widetilde\phia(a).
\]
For fixed $(s,a,s')$,
\[
g(s,a,s')-\widetilde g(s,a,s')
=
x(s,a)^\top\bigl(m(s')-\widetilde m(s')\bigr)
+
\bigl(x(s,a)-\widetilde x(s,a)\bigr)^\top\widetilde m(s').
\]
By Cauchy--Schwarz,
\[
|g(s,a,s')-\widetilde g(s,a,s')|
\le
\|x(s,a)\|_2\|m(s')-\widetilde m(s')\|_2
+
\|x(s,a)-\widetilde x(s,a)\|_2\|\widetilde m(s')\|_2.
\]
The Hadamard product bound gives
\[
\|x(s,a)\|_2
=
\|\phis(s)\odot\phia(a)\|_2
\le
\|\phis(s)\|_2\|\phia(a)\|_2
\le
B_\phi^2.
\]
Moreover,
\[
x(s,a)-\widetilde x(s,a)
=
(\phis(s)-\widetilde\phis(s))\odot\phia(a)
+
\widetilde\phis(s)\odot(\phia(a)-\widetilde\phia(a)).
\]
Using
\[
\|u\odot v\|_2\le \|v\|_\infty\|u\|_2
\le
\|v\|_2\|u\|_2,
\]
we obtain
\[
\|x(s,a)-\widetilde x(s,a)\|_2
\le
B_\phi\|\phis(s)-\widetilde\phis(s)\|_2
+
B_\phi\|\phia(a)-\widetilde\phia(a)\|_2.
\]
Taking the supremum over $(s,a,s')$ proves \eqref{eq:sup-bound}.
\end{proof}

\begin{proof}
We prove the two statements separately.
\paragraph{Part (i): additive decomposition.}
Choose component covers $\widehat\Fcal_s$, $\widehat\Fcal_a$, $\widehat\Fcal_m$
of $\Fcal_s,\Fcal_a,\Fcal_m$ at radii
\[
\epsilon_s:=\frac{\epsilon}{3B_{m,\infty}B_\phi},
\qquad
\epsilon_a:=\frac{\epsilon}{3B_{m,\infty}B_\phi},
\qquad
\epsilon_m:=\frac{\epsilon}{3B_\phi^2},
\]
respectively, under the corresponding uniform sup-$\ell_2$ metrics.
Take any
\[
g(s,a,s')=(\phis(s)\odot\phia(a))^\top m(s')
\in\Gcal_{\mathrm{fac}}.
\]
By the definitions of the three covers, choose
\[
\widehat\phis\in\widehat\Fcal_s,
\qquad
\widehat\phia\in\widehat\Fcal_a,
\qquad
\widehat m\in\widehat\Fcal_m
\]
such that
\[
\sup_s\|\phis(s)-\widehat\phis(s)\|_2\le \epsilon_s,
\qquad
\sup_a\|\phia(a)-\widehat\phia(a)\|_2\le \epsilon_a,
\qquad
\sup_{s'}\|m(s')-\widehat m(s')\|_2\le \epsilon_m.
\]
Define
\[
\widehat g(s,a,s')
=
(\widehat\phis(s)\odot\widehat\phia(a))^\top\widehat m(s').
\]
Applying Lemma~\ref{lem:cp-score-perturb} gives
\begin{align*}
\|g-\widehat g\|_\infty
&\le
B_\phi^2\epsilon_m
+
B_{m,\infty}B_\phi\epsilon_s
+
B_{m,\infty}B_\phi\epsilon_a        \\
&=
\frac{\epsilon}{3}
+
\frac{\epsilon}{3}
+
\frac{\epsilon}{3}
=
\epsilon.
\end{align*}
Thus the product class
\[
\widehat\Gcal_{\mathrm{fac}}
:=
\left\{
(s,a,s')\mapsto
(\widehat\phis(s)\odot\widehat\phia(a))^\top\widehat m(s')
:
\widehat\phis\in\widehat\Fcal_s,\
\widehat\phia\in\widehat\Fcal_a,\
\widehat m\in\widehat\Fcal_m
\right\}
\]
is an $\epsilon$-cover of $\Gcal_{\mathrm{fac}}$. Each element of
$\widehat\Gcal_{\mathrm{fac}}$ is determined by a triple from the three
component covers, so
\[
|\widehat\Gcal_{\mathrm{fac}}|
\le
|\widehat\Fcal_s|\,|\widehat\Fcal_a|\,|\widehat\Fcal_m|.
\]
Taking the smallest possible component covers, or equivalently taking the
infimum over such covers, yields
\[
\Ncal(\Gcal_{\mathrm{fac}},\epsilon)
\le
\Ncal(\Fcal_s,\epsilon_s)
\Ncal(\Fcal_a,\epsilon_a)
\Ncal(\Fcal_m,\epsilon_m).
\]
Taking logarithms and substituting the definitions of
$\epsilon_s,\epsilon_a,\epsilon_m$ gives \eqref{eq:covering-decomp}.
\paragraph{Part (ii): containment.}
Take any $g\in\Gcal_{\mathrm{fac}}$. Then, for some
$\phis\in\Fcal_s$, $\phia\in\Fcal_a$, and $m\in\Fcal_m$,
\[
g(s,a,s')
=
(\phis(s)\odot\phia(a))^\top m(s').
\]
By the assumption on $\Fcal_{sa}$,
\[
\varphi(s,a):=\phis(s)\odot\phia(a)
\]
belongs to $\Fcal_{sa}$. Taking $\mu:=m$, we get
\[
g(s,a,s')=\varphi(s,a)^\top\mu(s'),
\]
so $g\in\Gcal_{\mathrm{joint}}$. Hence
\[
\Gcal_{\mathrm{fac}}\subseteq\Gcal_{\mathrm{joint}}.
\]
The covering-number inequality follows because any $\epsilon$-cover of a set
also covers all of its subsets.
\end{proof}

\subsection{Lipschitz Instantiation: an Upper-Bound vs.\ Upper-Bound Comparison}
We now explain how the additive decomposition in \eqref{eq:covering-decomp}
translates into the stated entropy scaling under standard Lipschitz covering
bounds.
Assume that
\[
\Scal\subseteq\R^{d_s},
\qquad
\Acal\subseteq\R^{d_a}
\]
are compact, and that $\Fcal_s,\Fcal_a,\Fcal_m$ are uniformly bounded
Lipschitz classes. For $L$-Lipschitz maps from a compact set
$X\subseteq\R^p$ into a bounded subset of $\R^d$, the standard entropy bound
gives
\[
\log\Ncal(\Fcal,\epsilon)
\le
c\,d\bigl(L\diam(X)/\epsilon\bigr)^p
\]
for the relevant range of $\epsilon$; see
\citet[Theorem~5.7]{wainwright2019high}. The important point is that the
exponent is the dimension $p$ of the input domain.
Applying this bound to the three terms in \eqref{eq:covering-decomp} yields
finite constants $A_s,A_a,A_m$ such that
\[
\log\Ncal(\Gcal_{\mathrm{fac}},\epsilon)
\le
A_s\epsilon^{-d_s}
+
A_a\epsilon^{-d_a}
+
A_m\epsilon^{-d_s}.
\]
Therefore, for
\[
D_{\max}:=\max(d_s,d_a)
\]
and $0<\epsilon\le 1$,
\begin{equation}\label{eq:Gfac-lip-rate}
\log\Ncal(\Gcal_{\mathrm{fac}},\epsilon)
\le
A_{\mathrm{fac}}\epsilon^{-D_{\max}},
\end{equation}
where $A_{\mathrm{fac}}<\infty$ depends on the Lipschitz constants, diameters,
output dimension, and uniform radius bounds, but not on $\epsilon$.
For the joint class, the unrestricted state-action encoder is a Lipschitz class
on the product domain
\[
\Scal\times\Acal\subseteq\R^{d_s+d_a}.
\]
The same two-term covering argument for
\[
g(s,a,s')=\varphi(s,a)^\top\mu(s')
\]
gives
\[
\log\Ncal(\Gcal_{\mathrm{joint}},\epsilon)
\le
\log\Ncal\!\Bigl(\Fcal_{sa},\tfrac{\epsilon}{2B_{m,\infty}}\Bigr)
+
\log\Ncal\!\Bigl(\Fcal_m,\tfrac{\epsilon}{2B_{sa}}\Bigr),
\]
where $B_{sa}$ is a uniform bound on $\|\varphi(s,a)\|_2$. Applying the
Lipschitz entropy bound to $\Fcal_{sa}$ on the product domain and to $\Fcal_m$
on $\Scal$ yields finite constants $A_{sa},A_m'$ such that
\[
\log\Ncal(\Gcal_{\mathrm{joint}},\epsilon)
\le
A_{sa}\epsilon^{-(d_s+d_a)}
+
A_m'\epsilon^{-d_s}.
\]
Since $d_s\le d_s+d_a$, for $0<\epsilon\le1$ this implies
\begin{equation}\label{eq:Gjoint-lip-rate}
\log\Ncal(\Gcal_{\mathrm{joint}},\epsilon)
\le
A_{\mathrm{joint}}\epsilon^{-(d_s+d_a)}
\end{equation}
for a finite constant $A_{\mathrm{joint}}$ independent of $\epsilon$.
Combining \eqref{eq:Gfac-lip-rate} and \eqref{eq:Gjoint-lip-rate}, the
factored and joint upper bounds scale as
\[
\log\Ncal(\Gcal_{\mathrm{fac}},\epsilon)
\lesssim
\epsilon^{-\max(d_s,d_a)},
\qquad
\log\Ncal(\Gcal_{\mathrm{joint}},\epsilon)
\lesssim
\epsilon^{-(d_s+d_a)}.
\]
This is an upper-bound comparison: under the same covering template, the factored entropy certificate depends on $\max(d_s,d_a)$ while the joint certificate depends on $d_s+d_a$. We discuss the absence of a matching lower bound in Remark~\ref{rmk:not-minimax}.
\subsection{Rademacher Monotonicity Component of Theorem~\ref{thm:complexity}(ii)}
\label{app:proof-rademacher}

Theorem~\ref{thm:complexity}(ii) states that the inclusion $\Gcal_{\mathrm{fac}}\subseteq\Gcal_{\mathrm{joint}}$ propagates to both covering numbers and Rademacher complexities at every sample size. The covering-number half follows from monotonicity of $\Ncal(\cdot,\epsilon)$ under set inclusion (Appendix~\ref{app:proof-complexity}, Part~(ii)). The Rademacher half follows from a simpler argument than the covering bound: it requires only set inclusion, not entropy.

\begin{proposition*}[Rademacher inequality from Theorem~\ref{thm:complexity}(ii)]
Under the containment hypothesis of Theorem~\ref{thm:complexity}(ii), for every $n$ and every i.i.d.\ sample $z_1,\dots,z_n$,
\[
\Rcal_n(\Gcal_{\mathrm{fac}}) \;\le\; \Rcal_n(\Gcal_{\mathrm{joint}}).
\]
\end{proposition*}

\begin{proof}
The inclusion $\Gcal_{\mathrm{fac}}\subseteq\Gcal_{\mathrm{joint}}$ was established in Theorem~\ref{thm:complexity}(ii). For every fixed sample $S$ and every sign realization $\sigma$,
\[
\sup_{g\in\Gcal_{\mathrm{fac}}}
\left|
\frac{1}{n}\sum_{i=1}^n\sigma_i\, g(z_i)
\right|
\;\le\;
\sup_{g\in\Gcal_{\mathrm{joint}}}
\left|
\frac{1}{n}\sum_{i=1}^n\sigma_i\, g(z_i)
\right|,
\]
because the supremum over a subset is no larger than the supremum over the superset. Taking expectation over $\sigma$ gives $\widehat\Rcal_S(\Gcal_{\mathrm{fac}})\le\widehat\Rcal_S(\Gcal_{\mathrm{joint}})$; taking expectation over $S$ gives the stated $\Rcal_n$ inequality.
\end{proof}

\paragraph{Remark (Why not via Dudley?)}\label{rmk:no-dudley}
An alternative would be to bound both Rademacher complexities via Dudley's entropy integral and observe that $\Ncal(\Gcal_{\mathrm{fac}},\epsilon)\le\Ncal(\Gcal_{\mathrm{joint}},\epsilon)$ propagates through the integral. This argument only bounds the \emph{upper bounds} on the two complexities and does not establish the ordering of the complexities themselves: the joint Rademacher complexity could in principle be much smaller than its Dudley upper bound. Set inclusion gives the stronger statement.

\subsection{Proof of Proposition~\ref{prop:finite-sample}: Excess Risk for the Spectral $L_2$ Objective}
\label{app:proof-finite-sample}

We now address estimation. The objective trained in practice is the
noise-contrastive estimator (NCE)~\citep{gutmann2010noise,ma2018noise,zhang2022making,gao2025spectral} for the transition-ratio score, but the
theoretical guarantees that bound representation error are derived in the
spectral $L_2$ framework of \citet{ren2022spectral}; we follow the same
convention and analyze the $L_2$ ERM here. The relationship between NCE
optimization and $L_2$ population structure is discussed in
Remark~\ref{rmk:nce} and is \emph{not} part of the formal claim of this
proposition.

Let
\[
z=(s,a,s')
\]
denote one observation, where
\[
(s,a)\sim \rho,
\qquad
s'\sim P(\cdot\mid s,a).
\]
For an i.i.d.\ sample
\[
z_i=(s_i,a_i,s'_i),
\qquad i=1,\ldots,n,
\]
define the population $L_2$ risk
\begin{equation}\label{eq:L2-risk-pop}
L(g)
:=
\E_\rho\!\left[
\int_\Scal
\bigl(P(s'\mid s,a)-g(s,a,s')\bigr)^2
\,d\nu(s')
\right].
\end{equation}
Also define
\begin{equation}\label{eq:CP-constant}
C_P
:=
\E_\rho\!\left[
\int_\Scal
P(s'\mid s,a)^2
\,d\nu(s')
\right],
\end{equation}
and the centered population risk
\begin{equation}\label{eq:centered-risk}
\bar L(g)
:=
L(g)-C_P.
\end{equation}
The empirical objective is
\begin{equation}\label{eq:empirical-risk}
\hat L_n(g)
:=
\frac{1}{n}
\sum_{i=1}^n
\ell(g;z_i),
\end{equation}
where the per-sample loss is
\begin{equation}\label{eq:per-sample-loss}
\ell(g;z)
:=
-2g(s,a,s')
+
\int_\Scal g^2(s,a,u)\,d\nu(u).
\end{equation}
We assume $C_P<\infty$ throughout. This holds, for example, whenever
$P(\cdot\mid s,a)\in L^2(\nu)$ uniformly on the support of $\rho$. In particular,
under CP realizability with a uniformly bounded density representation
$|P(s'\mid s,a)|\le B_g$ and with $\nu(\Scal)=1$, we have
\[
C_P
=
\E_\rho\!\left[
\int_\Scal P(s'\mid s,a)^2\,d\nu(s')
\right]
\le
\E_\rho\!\left[
\int_\Scal B_g^2\,d\nu(s')
\right]
=
B_g^2.
\]

We use the Rademacher-complexity convention fixed in
Appendix~\ref{app:notation}. For the quadratic class below, the sample is the
projected sample $(s_i,a_i)_{i=1}^n$.

\begin{lemma}[Covering transfer to the integrated quadratic class]
\label{lem:quadratic-cover-transfer}
Assume $\nu(\Scal)=1$ and $|g(s,a,s')|\le B_g$ for all
$g\in\Gcal$ and all $(s,a,s')$. Define
\[
h_g(s,a)
:=
\int_\Scal g^2(s,a,u)\,d\nu(u),
\qquad
\Hcal_\Gcal:=\{h_g:g\in\Gcal\}.
\]
Then, for every $\epsilon>0$,
\begin{equation}\label{eq:quadratic-cover-transfer}
\Ncal(\Hcal_\Gcal,\epsilon)
\le
\Ncal\!\left(\Gcal,\frac{\epsilon}{2B_g}\right).
\end{equation}
\end{lemma}
\begin{proof}
If $\|g-g'\|_\infty\le \eta$, then for every $(s,a)$,
\begin{align*}
|h_g(s,a)-h_{g'}(s,a)|
&=
\left|
\int_\Scal
\bigl(g^2(s,a,u)-g'^2(s,a,u)\bigr)
\,d\nu(u)
\right|                                           \\
&\le
\int_\Scal
|g(s,a,u)-g'(s,a,u)|\,|g(s,a,u)+g'(s,a,u)|
\,d\nu(u)                                        \\
&\le
\int_\Scal
\eta\cdot 2B_g
\,d\nu(u)
=
2B_g\eta.
\end{align*}
Thus any $\eta$-cover of $\Gcal$ in sup norm induces a $2B_g\eta$-cover of
$\Hcal_\Gcal$. Setting $\eta=\epsilon/(2B_g)$ proves
\eqref{eq:quadratic-cover-transfer}.
\end{proof}

\begin{proposition*}[Restatement of Proposition~\ref{prop:finite-sample}]
Let $\Gcal$ be either $\Gcal_{\mathrm{fac}}$ or $\Gcal_{\mathrm{joint}}$, and
assume
\[
\sup_{s,a,s'} |g(s,a,s')|\le B_g
\qquad
\text{for all }g\in\Gcal.
\]
Assume also that $C_P<\infty$. Let
\[
\hat g_n\in\arg\min_{g\in\Gcal}\hat L_n(g)
\]
and
\[
g^*_\Gcal\in\arg\min_{g\in\Gcal}L(g),
\]
both assumed to exist. If the empirical minimizer exists only approximately,
namely if
\[
\hat L_n(\hat g_n)
\le
\inf_{g\in\Gcal}\hat L_n(g)+\eta,
\]
then the same bound below holds with an additional $+\eta$ term on the
right-hand side.

Define the induced quadratic class
\[
\Hcal_\Gcal
:=
\left\{
h_g:\Scal\times\Acal\to\R
\;:\;
h_g(s,a)
=
\int_\Scal g^2(s,a,u)\,d\nu(u),
\quad
g\in\Gcal
\right\}.
\]
Then, with probability at least $1-\delta$,
\begin{equation}\label{eq:finite-sample}
L(\hat g_n)
\le
L(g^*_\Gcal)
+
8\,\Rcal_n(\Gcal)
+
4\,\Rcal_n(\Hcal_\Gcal)
+
2(2B_g+B_g^2)
\sqrt{\frac{2\log(1/\delta)}{n}}.
\end{equation}
Here $\Rcal_n(\Hcal_\Gcal)$ is computed on the projected sample
$(s_i,a_i)_{i=1}^n$.

Moreover, both Rademacher-complexity terms admit Dudley upper bounds in terms of
the covering number of $\Gcal$. In particular,
\begin{equation}\label{eq:H-cover}
\Ncal(\Hcal_\Gcal,\epsilon)
\le
\Ncal\!\left(\Gcal,\frac{\epsilon}{2B_g}\right).
\end{equation}
Thus the bound \eqref{eq:finite-sample} can be expressed entirely in terms of
the metric entropy of $\Gcal$.
\end{proposition*}

\begin{proof}
First, expanding the square in $L(g)$ gives
\begin{align*}
L(g)
&=
\E_\rho\!\int_\Scal
\bigl(P(s'\mid s,a)-g(s,a,s')\bigr)^2\,d\nu(s')        \\
&=
C_P
-
2\E_\rho\!\int_\Scal
P(s'\mid s,a)g(s,a,s')\,d\nu(s')
+
\E_\rho\!\int_\Scal g^2(s,a,s')\,d\nu(s').
\end{align*}
Since $s'\sim P(\cdot\mid s,a)$ conditional on $(s,a)$,
\[
\E[g(s,a,s')]
=
\E_\rho\!\int_\Scal
P(s'\mid s,a)g(s,a,s')\,d\nu(s').
\]
Therefore
\begin{equation}\label{eq:barL-is-Eell}
\bar L(g)=L(g)-C_P=\E[\ell(g;z)].
\end{equation}
In particular, $\hat L_n(g)$ is an unbiased empirical estimate of the centered
risk $\bar L(g)$, and $L$ and $\bar L$ have the same minimizers because they
differ only by the constant $C_P$.
Using empirical optimality of $\hat g_n$,
\begin{align}
L(\hat g_n)-L(g^*_\Gcal)
&=
\bar L(\hat g_n)-\bar L(g^*_\Gcal)                         \notag\\
&=
[\bar L(\hat g_n)-\hat L_n(\hat g_n)]
+
[\hat L_n(\hat g_n)-\hat L_n(g^*_\Gcal)]
+
[\hat L_n(g^*_\Gcal)-\bar L(g^*_\Gcal)]                    \notag\\
&\le
2\sup_{g\in\Gcal}
|\bar L(g)-\hat L_n(g)|.
\label{eq:excess-decomp}
\end{align}
If $\hat g_n$ is only an $\eta$-approximate empirical minimizer, the same
argument adds an extra $+\eta$ term.
Let
\[
\Lcal_\Gcal
:=
\{\ell_g:z\mapsto \ell(g;z):g\in\Gcal\}.
\]
By the standard symmetrization inequality for the centered empirical process~\citep{bartlett2002rademacher,wainwright2019high},
\begin{equation}\label{eq:sym}
\E\sup_{g\in\Gcal}
|\bar L(g)-\hat L_n(g)|
\le
2\Rcal_n(\Lcal_\Gcal).
\end{equation}
Now write
\[
\ell(g;z)
=
-2g(s,a,s')+h_g(s,a),
\qquad
h_g(s,a):=\int_\Scal g^2(s,a,u)\,d\nu(u).
\]
By subadditivity and homogeneity of Rademacher complexity,
\begin{equation}\label{eq:loss-decomp}
\Rcal_n(\Lcal_\Gcal)
\le
2\Rcal_n(\Gcal)+\Rcal_n(\Hcal_\Gcal).
\end{equation}
Combining \eqref{eq:sym} and \eqref{eq:loss-decomp} gives
\begin{equation}\label{eq:expected-uniform-bound}
\E\sup_{g\in\Gcal}
|\bar L(g)-\hat L_n(g)|
\le
4\Rcal_n(\Gcal)+2\Rcal_n(\Hcal_\Gcal).
\end{equation}
The quadratic term is not controlled by the usual pointwise contraction
principle, since
\[
h_g(s,a)=\int_\Scal g^2(s,a,u)\,d\nu(u)
\]
integrates over all possible next states $u$ rather than applying a scalar
function to the observed value $g(s_i,a_i,s'_i)$. Instead, Lemma
\ref{lem:quadratic-cover-transfer} gives
\[
\Ncal(\Hcal_\Gcal,\epsilon)
\le
\Ncal\!\left(\Gcal,\frac{\epsilon}{2B_g}\right),
\]
which is the entropy estimate \eqref{eq:H-cover} stated in the proposition.
Dudley's entropy integral~\citep{dudley1967sizes,wainwright2019high} therefore bounds both $\Rcal_n(\Gcal)$ and
$\Rcal_n(\Hcal_\Gcal)$ in terms of the covering number of $\Gcal$, with the
second bound using only the radius rescaling above.
It remains to pass from expectation to high probability. Define
\[
F(z_1,\ldots,z_n)
:=
\sup_{g\in\Gcal}
|\bar L(g)-\hat L_n(g)|.
\]
Since $|g|\le B_g$ and $\nu(\Scal)=1$,
\[
|\ell(g;z)|
\le
2|g(s,a,s')|
+
\int_\Scal g^2(s,a,u)\,d\nu(u)
\le
2B_g+B_g^2
=:M_g.
\]
Changing one sample point changes $\hat L_n(g)$ by at most $2M_g/n$ uniformly
over $g$, and therefore changes $F$ by at most $2M_g/n$. McDiarmid's inequality~\citep{mcdiarmid1989method}
yields, with probability at least $1-\delta$,
\[
F
\le
\E F
+
M_g\sqrt{\frac{2\log(1/\delta)}{n}}.
\]
Combining this display with \eqref{eq:expected-uniform-bound} and then using
\eqref{eq:excess-decomp}, we obtain
\begin{align*}
L(\hat g_n)-L(g^*_\Gcal)
&\le
2F                                                        \\
&\le
8\Rcal_n(\Gcal)
+
4\Rcal_n(\Hcal_\Gcal)
+
2M_g\sqrt{\frac{2\log(1/\delta)}{n}}.
\end{align*}
Substituting $M_g=2B_g+B_g^2$ proves \eqref{eq:finite-sample}.
\end{proof}

\paragraph{Remark (On the constants)}\label{rmk:constants}
The prefactors in \eqref{eq:finite-sample} come directly from the proof. The
factor $2$ in the excess-risk decomposition \eqref{eq:excess-decomp}, the factor
$2$ from symmetrization \eqref{eq:sym}, and the decomposition
\[
\Rcal_n(\Lcal_\Gcal)
\le
2\,\Rcal_n(\Gcal)+\Rcal_n(\Hcal_\Gcal)
\]
together give the coefficients $8$ and $4$ in front of
$\Rcal_n(\Gcal)$ and $\Rcal_n(\Hcal_\Gcal)$, respectively. The tail coefficient
comes from the uniform bound
\[
|\ell(g;z)|\le 2B_g+B_g^2
\]
and the bounded-differences constant
\[
\frac{2(2B_g+B_g^2)}{n}.
\]
Thus these constants are explicit and problem-dependent. We avoid the phrase
``universal constant'' here because the dependence on $B_g$ is essential: it
enters both through the per-sample loss bound and through the covering-number
rescaling
\[
\Ncal(\Hcal_\Gcal,\epsilon)
\le
\Ncal\!\left(\Gcal,\frac{\epsilon}{2B_g}\right).
\]

\paragraph{Remark (Relationship to the NCE estimator)}\label{rmk:nce}
The ranking-based NCE objective of~\citet{ma2018noise} used by FaStR is consistent with maximum likelihood in the limit of infinitely many negative samples, but is not identical to the $L_2$ spectral ERM analyzed here. Establishing a sharp finite-sample equivalence between the two estimators is separate from the claim proved above.

\subsection{Proof of Proposition~\ref{prop:sample-sep}: Bound-Implied Sufficient Sample Size}
\label{app:proof-sample-sep}

We combine the additive covering bound from Theorem~\ref{thm:complexity} with
the excess-risk template from Proposition~\ref{prop:finite-sample}. The goal is
to compare the sample sizes that this particular uniform-convergence analysis
\emph{certifies} as sufficient for the factored class and for the joint class.

Throughout this section, the word ``certificate'' means an upper bound produced
by the displayed finite-sample theorem. Thus the comparison below is a comparison
between two sufficient sample sizes obtained by inverting two upper bounds. It is
not a comparison between the true minimax sample complexities of the two
statistical problems. In particular, the joint-class bound below is the standard
uniform-convergence certificate for an unrestricted Lipschitz state-action
encoder. It is not a lower bound and may be loose for structured subclasses.

\begin{lemma}[Polynomial entropy implies a Rademacher rate]
\label{lem:poly-entropy-rademacher}
Let $\Fcal$ be a uniformly bounded real-valued function class with
\[
\operatorname{diam}_\infty(\Fcal)
:=
\sup_{f,f'\in\Fcal}\|f-f'\|_\infty
\le 2R.
\]
Assume that, for some $A<\infty$ and $D>2$,
\[
\log\Ncal(\Fcal,\tau)\le A\tau^{-D},
\qquad
0<\tau\le R.
\]
Then there exists a constant $K(A,D,R)<\infty$, independent of $n$, such that
\begin{equation}\label{eq:entropy-to-rademacher}
\Rcal_n(\Fcal)
\le
K(A,D,R)n^{-1/D}.
\end{equation}
\end{lemma}
\begin{proof}
Dudley's entropy integral, together with the fact that a sup-norm cover is also
an empirical-$L_2$ cover, gives for any $\alpha\in(0,R]$,
\[
\Rcal_n(\Fcal)
\le
4\alpha
+
\frac{12}{\sqrt n}
\int_\alpha^R
\sqrt{\log\Ncal(\Fcal,\tau)}
\,d\tau.
\]
Using the polynomial entropy bound,
\[
\Rcal_n(\Fcal)
\le
4\alpha
+
\frac{12\sqrt A}{\sqrt n}
\int_\alpha^R \tau^{-D/2}\,d\tau.
\]
Since $D>2$,
\[
\int_\alpha^R \tau^{-D/2}\,d\tau
\le
\frac{2}{D-2}\alpha^{1-D/2}.
\]
Choose $\alpha=Rn^{-1/D}$, which lies in $(0,R]$ for $n\ge1$. Then
\[
4\alpha=4Rn^{-1/D},
\]
and
\[
\frac{1}{\sqrt n}\alpha^{1-D/2}
=
\frac{1}{\sqrt n}(Rn^{-1/D})^{1-D/2}
=
R^{1-D/2}n^{-1/D}.
\]
Therefore
\[
\Rcal_n(\Fcal)
\le
\left(
4R+\frac{24\sqrt A}{D-2}R^{1-D/2}
\right)n^{-1/D}.
\]
The coefficient is finite and independent of $n$, proving the claim.
\end{proof}

\begin{proposition*}[Restatement of Proposition~\ref{prop:sample-sep}]
Assume CP realizability, that is, $\epsilon_{\mathrm{CP}}=0$. Also assume the
containment condition of Theorem~\ref{thm:complexity}(ii), so that the true CP
transition model lies in both $\Gcal_{\mathrm{fac}}$ and
$\Gcal_{\mathrm{joint}}$.

Fix $\delta\in(0,1)$. Let
\[
\Fcal_s,\qquad
\Fcal_a,\qquad
\Fcal_m,\qquad
\Fcal_{sa}
\]
be $L$-Lipschitz classes on compact domains. Let
\[
D_{\max}:=\max(d_s,d_a),
\qquad
D_{\mathrm{sum}}:=d_s+d_a,
\]
and assume
\[
D_{\max}\ge 3.
\]
Let
\[
\hat g_n^{\,\mathrm{fac}}
\in
\arg\min_{g\in\Gcal_{\mathrm{fac}}}\hat L_n(g),
\qquad
\hat g_n^{\,\mathrm{joint}}
\in
\arg\min_{g\in\Gcal_{\mathrm{joint}}}\hat L_n(g)
\]
be empirical minimizers for the two classes.

Then, with probability at least $1-\delta$, the high-probability bound of
Proposition~\ref{prop:finite-sample} yields the certificates
\begin{align}
L(\hat g_n^{\,\mathrm{fac}})
&\le
C_{\mathrm{fac}}(\delta)\,n^{-1/D_{\max}},
&
L(\hat g_n^{\,\mathrm{joint}})
&\le
C_{\mathrm{joint}}(\delta)\,n^{-1/D_{\mathrm{sum}}},
\label{eq:rates}
\end{align}
where the constants
\[
C_{\mathrm{fac}}(\delta),
\qquad
C_{\mathrm{joint}}(\delta)
\]
are finite, positive, independent of $n$ and of the target accuracy
$\epsilon$, and depend polynomially on the problem parameters
\[
B_g,\quad B_\phi,\quad B_{m,\infty},\quad L,\quad
\diam(\Scal),\quad \diam(\Acal),\quad d,
\]
and logarithmically on $1/\delta$, with fixed dependence on the dimensions
$d_s,d_a$.

Define the bound-implied sufficient sample sizes
\[
n_{\mathrm{fac}}^{\mathrm{suff}}(\epsilon)
:=
\inf\left\{
n\in\mathbb{N}:
C_{\mathrm{fac}}(\delta)\,n^{-1/D_{\max}}
\le
\epsilon^2
\right\},
\]
and
\[
n_{\mathrm{joint}}^{\mathrm{suff}}(\epsilon)
:=
\inf\left\{
n\in\mathbb{N}:
C_{\mathrm{joint}}(\delta)\,n^{-1/D_{\mathrm{sum}}}
\le
\epsilon^2
\right\}.
\]
These are the smallest sample sizes at which the certificates in
\eqref{eq:rates} guarantee representation error at most $\epsilon^2$. Then
\begin{equation}\label{eq:sample-ratio}
\frac{
n_{\mathrm{joint}}^{\mathrm{suff}}(\epsilon)
}{
n_{\mathrm{fac}}^{\mathrm{suff}}(\epsilon)
}
=
\Theta\!\left(
\epsilon^{-2\min(d_s,d_a)}
\right)
\qquad
\text{as }\epsilon\to 0.
\end{equation}
\end{proposition*}

\begin{proof}
Under CP realizability, there exists
\[
g_{\mathrm{CP}}\in\Gcal_{\mathrm{fac}}
\]
such that
\[
g_{\mathrm{CP}}(s,a,s')=P(s'\mid s,a).
\]
Hence
\[
L(g_{\mathrm{CP}})
=
\E_\rho\!\int_\Scal
\bigl(P(s'\mid s,a)-g_{\mathrm{CP}}(s,a,s')\bigr)^2
\,d\nu(s')
=
0.
\]
Since $L(g)\ge0$ for all $g$, we have
\[
L(g^*_{\Gcal_{\mathrm{fac}}})=0.
\]
The containment condition gives
\[
\Gcal_{\mathrm{fac}}\subseteq\Gcal_{\mathrm{joint}},
\]
so the same $g_{\mathrm{CP}}$ also belongs to $\Gcal_{\mathrm{joint}}$. Thus
\[
L(g^*_{\Gcal_{\mathrm{joint}}})=0.
\]
Therefore the finite-sample certificates contain only estimation and
concentration terms.
By Theorem~\ref{thm:complexity} and the Lipschitz entropy bound, there is a
constant $A_{\mathrm{fac}}<\infty$ such that
\[
\log\Ncal(\Gcal_{\mathrm{fac}},\tau)
\le
A_{\mathrm{fac}}\tau^{-D_{\max}},
\qquad
D_{\max}:=\max(d_s,d_a).
\]
For the unrestricted joint Lipschitz class, the standard product-domain entropy
bound gives
\[
\log\Ncal(\Gcal_{\mathrm{joint}},\tau)
\le
A_{\mathrm{joint}}\tau^{-D_{\mathrm{sum}}},
\qquad
D_{\mathrm{sum}}:=d_s+d_a.
\]
Moreover, by Lemma~\ref{lem:quadratic-cover-transfer}, the associated quadratic
classes have the same entropy exponents up to constants:
\[
\log\Ncal(\Hcal_{\Gcal_{\mathrm{fac}}},\tau)
\le
A_{\mathrm{fac},H}\tau^{-D_{\max}},
\]
and
\[
\log\Ncal(\Hcal_{\Gcal_{\mathrm{joint}}},\tau)
\le
A_{\mathrm{joint},H}\tau^{-D_{\mathrm{sum}}},
\]
for suitable finite constants $A_{\mathrm{fac},H}$ and
$A_{\mathrm{joint},H}$.
Since $D_{\max}\ge3$ and
\[
D_{\mathrm{sum}}=d_s+d_a\ge D_{\max},
\]
both entropy exponents are larger than $2$. Applying
Lemma~\ref{lem:poly-entropy-rademacher} gives constants
$K_{\mathrm{fac}},K_{\mathrm{joint}}<\infty$, independent of $n$ and
$\epsilon$, such that
\[
\Rcal_n(\Gcal_{\mathrm{fac}})
+
\Rcal_n(\Hcal_{\Gcal_{\mathrm{fac}}})
\le
K_{\mathrm{fac}}n^{-1/D_{\max}},
\]
and
\[
\Rcal_n(\Gcal_{\mathrm{joint}})
+
\Rcal_n(\Hcal_{\Gcal_{\mathrm{joint}}})
\le
K_{\mathrm{joint}}n^{-1/D_{\mathrm{sum}}}.
\]
Apply Proposition~\ref{prop:finite-sample} to the two classes with confidence
levels $\delta/2$ and take a union bound. With probability at least
$1-\delta$, both bounds hold simultaneously. Using the zero oracle risks above,
\[
L(\hat g_n^{\,\mathrm{fac}})
\le
K_{\mathrm{fac}}'n^{-1/D_{\max}}
+
2(2B_g+B_g^2)
\sqrt{\frac{2\log(2/\delta)}{n}},
\]
and
\[
L(\hat g_n^{\,\mathrm{joint}})
\le
K_{\mathrm{joint}}'n^{-1/D_{\mathrm{sum}}}
+
2(2B_g+B_g^2)
\sqrt{\frac{2\log(2/\delta)}{n}},
\]
for finite constants $K_{\mathrm{fac}}'$ and $K_{\mathrm{joint}}'$.
Because $D_{\max}\ge3$ and $D_{\mathrm{sum}}\ge D_{\max}$,
\[
n^{-1/2}\le n^{-1/D_{\max}},
\qquad
n^{-1/2}\le n^{-1/D_{\mathrm{sum}}},
\qquad n\ge1.
\]
Absorbing the concentration terms into the constants yields
\begin{align}
L(\hat g_n^{\,\mathrm{fac}})
&\le
C_{\mathrm{fac}}(\delta)n^{-1/D_{\max}},
&
L(\hat g_n^{\,\mathrm{joint}})
&\le
C_{\mathrm{joint}}(\delta)n^{-1/D_{\mathrm{sum}}},
\label{eq:rates-proof}
\end{align}
which proves the certificates in \eqref{eq:rates}. The constants are finite,
positive, independent of $n$ and $\epsilon$, and have the parameter dependence
stated in the proposition.
It remains to invert the certificates. The factored certificate guarantees
\[
L(\hat g_n^{\,\mathrm{fac}})\le\epsilon^2
\]
whenever
\[
C_{\mathrm{fac}}(\delta)n^{-1/D_{\max}}\le\epsilon^2.
\]
This is equivalent to
\[
n
\ge
C_{\mathrm{fac}}(\delta)^{D_{\max}}\epsilon^{-2D_{\max}}.
\]
Similarly, the joint certificate guarantees
\[
L(\hat g_n^{\,\mathrm{joint}})\le\epsilon^2
\]
whenever
\[
n
\ge
C_{\mathrm{joint}}(\delta)^{D_{\mathrm{sum}}}
\epsilon^{-2D_{\mathrm{sum}}}.
\]
If sample sizes are required to be integers, ceilings change these thresholds by
a multiplicative $1+o(1)$ factor as $\epsilon\to0$. Hence
\[
\frac{
n_{\mathrm{joint}}^{\mathrm{suff}}(\epsilon)
}{
n_{\mathrm{fac}}^{\mathrm{suff}}(\epsilon)
}
=
\frac{
C_{\mathrm{joint}}(\delta)^{D_{\mathrm{sum}}}
}{
C_{\mathrm{fac}}(\delta)^{D_{\max}}
}
\epsilon^{-2(D_{\mathrm{sum}}-D_{\max})}
(1+o(1)).
\]
Finally,
\[
D_{\mathrm{sum}}-D_{\max}
=
d_s+d_a-\max(d_s,d_a)
=
\min(d_s,d_a).
\]
Therefore
\[
\frac{
n_{\mathrm{joint}}^{\mathrm{suff}}(\epsilon)
}{
n_{\mathrm{fac}}^{\mathrm{suff}}(\epsilon)
}
=
\Theta\!\left(
\epsilon^{-2\min(d_s,d_a)}
\right),
\]
as claimed. Since the thresholds were obtained by inverting upper-bound
certificates, this is an algebraic comparison of certified sufficient sample
sizes, not a minimax lower bound.
\end{proof}

\paragraph{Remark (Scope and interpretation of the sample-size comparison)}\label{rmk:not-minimax}
Equation~\eqref{eq:sample-ratio} compares the sample sizes that the same
uniform-convergence template certifies as sufficient for the two function
classes. The comparison is therefore between
\[
C_{\mathrm{fac}}(\delta)n^{-1/D_{\max}}
\]
and
\[
C_{\mathrm{joint}}(\delta)n^{-1/D_{\mathrm{sum}}},
\]
not between the true optimal risks achievable by all possible algorithms.

In particular, \eqref{eq:sample-ratio} does \emph{not} prove that learning the
joint class has minimax sample complexity
\[
\Omega(\epsilon^{-2D_{\mathrm{sum}}}).
\]
Such a statement would require a matching lower bound. A matching lower bound
would in turn require a packing construction showing that
$\Gcal_{\mathrm{joint}}$ contains many well-separated functions at scale
$\epsilon$. That kind of packing argument generally requires an additional
richness assumption on $\Fcal_{sa}$, ensuring that the joint state-action encoder
can realize a sufficiently large family of distinct functions over the full
$d_s+d_a$ dimensional product domain.

We do not impose such a richness assumption here. Therefore, the conservative
interpretation is the following: under the same generic uniform-convergence
analysis, the factored class receives a certificate with exponent
$D_{\max}=\max(d_s,d_a)$, whereas the unrestricted joint Lipschitz class receives
the standard certificate with exponent $D_{\mathrm{sum}}=d_s+d_a$. Inverting
these two certificates yields a polynomially smaller bound-implied sufficient
sample size for the factored class. Whether a more refined, algorithm-specific
analysis of the joint class can beat the generic uniform-convergence certificate
is a separate question. The empirical results in the main paper are consistent
with this conservative reading.

\paragraph{Remark (Non-realizable case)}\label{rmk:nonrealizable}
When $\epsilon_{\mathrm{CP}}>0$, the true transition model is not represented
exactly by the CP-factored class. In that case, the factored class has an
irreducible approximation term. If
\[
\inf_{g\in\Gcal_{\mathrm{fac}}} L(g)
\le
\epsilon_{\mathrm{CP}}^2,
\]
then the same finite-sample argument gives the factored certificate
\[
L(\hat g_n^{\,\mathrm{fac}})
\le
\epsilon_{\mathrm{CP}}^2
+
C_{\mathrm{fac}}(\delta)n^{-1/D_{\max}}.
\]
If the joint class still contains the true transition model, then its oracle
risk remains zero, and its certificate is
\[
L(\hat g_n^{\,\mathrm{joint}})
\le
C_{\mathrm{joint}}(\delta)n^{-1/D_{\mathrm{sum}}}.
\]
Thus the non-realizable setting involves a bias--variance tradeoff. The factored
class has the smaller entropy exponent and hence the faster estimation term, but
it may plateau at the approximation level $\epsilon_{\mathrm{CP}}^2$. For
moderate sample sizes and small CP approximation error, the factored certificate
can still be sharper; for sufficiently large sample sizes, if the joint class is
well-specified and the factored class is not, the irreducible bias term may
dominate the factored bound. This is why the exact-realizability result above is
stated separately.

\paragraph{Remark (Same-$n$ comparison without Lipschitz instantiation)}\label{rmk:same-n}
The sample-ratio in~\eqref{eq:sample-ratio} relies on the Lipschitz covering bounds, which give the explicit dimensional exponents $D_{\max}$ and $D_{\mathrm{sum}}$. A weaker but more robust statement is available without any Lipschitz instantiation: under CP realizability and the containment condition of Theorem~\ref{thm:complexity}(ii), at every fixed $n$, the certified upper bound on $L^2(\nu)$ representation error from Proposition~\ref{prop:finite-sample} is no larger for the factored class than for the joint class.

The argument uses only set inclusion. Appendix~\ref{app:proof-rademacher} gives $\Rcal_n(\Gcal_{\mathrm{fac}})\le\Rcal_n(\Gcal_{\mathrm{joint}})$ from $\Gcal_{\mathrm{fac}}\subseteq\Gcal_{\mathrm{joint}}$. The induced quadratic class $\Hcal_\Gcal=\{(s,a)\mapsto\int g^2(s,a,u)\,d\nu(u): g\in\Gcal\}$ inherits this monotonicity, since $h_g$ is determined by $g$ alone, so $\Rcal_n(\Hcal_{\Gcal_{\mathrm{fac}}})\le\Rcal_n(\Hcal_{\Gcal_{\mathrm{joint}}})$. Under CP realizability both approximation terms $L(g^*_\Gcal)$ in Proposition~\ref{prop:finite-sample} equal zero; with the shared envelope $B_g$ giving the same McDiarmid tail, the certificate~\eqref{eq:finite-sample} yields a no-larger upper bound on $L(\hat g_n)$ for the factored class at the same $n$. This certifies an ordering of two upper bounds, not a minimax separation.

\subsection{Conditional Implication for LSVI-UCB}
\label{app:lsvi-conditional}

The representation results above control an average transition-density error in
$L^2(\nu)$. This is the right object for spectral representation learning, but
LSVI-UCB requires a stronger condition: the Bellman backup must be uniformly
approximable as a linear function of the feature used by the algorithm. We state
this downstream implication conditionally. The point is not to derive this
condition from $L^2(\nu)$ error alone, but to record what the factored structure
would buy once the learned feature satisfies the standard approximate-linear
condition required by LSVI-UCB.

\begin{definition}[$(\eta,\Vcal,W)$-approximate Bellman linearity]
\label{def:approx-bellman-linearity}
A feature map $\widehat\psi:\Scal\times\Acal\to\R^d$ satisfies
$(\eta,\Vcal,W)$-approximate Bellman linearity if for every $V\in\Vcal$ there
exists a vector $w_V\in\R^d$ with $\|w_V\|_2\le W$ such that
\begin{equation}\label{eq:approx-bellman-linearity}
\sup_{(s,a)\in\Scal\times\Acal}
\left|
r(s,a)
+
\gamma\E_{s'\sim P(\cdot\mid s,a)}[V(s')]
-
\widehat\psi(s,a)^\top w_V
\right|
\le
\eta.
\end{equation}
Here $\Vcal$ can be taken as the value-function class generated by the
LSVI-UCB dynamic program. A stronger but more conservative version is obtained
by taking $\Vcal$ to be all bounded functions with
$\|V\|_\infty\le R_{\max}/(1-\gamma)$.
\end{definition}

\begin{proposition}[Conditional LSVI-UCB compatibility]
\label{prop:conditional-lsvi}
Let $\widehat\psi$ be a learned $d$-dimensional feature map, with
$\sup_{s,a}\|\widehat\psi(s,a)\|_2\le B_\psi$. Suppose that
$\widehat\psi$ satisfies $(\eta,\Vcal_{\mathrm{LSVI}},W)$-approximate Bellman
linearity in the sense of Definition~\ref{def:approx-bellman-linearity}.
Assume that LSVI-UCB is run with $\widehat\psi$, the standard elliptical bonus,
and the usual discounted-to-effective-horizon reduction
$H=\widetilde O((1-\gamma)^{-1})$.

If
\[
\eta \le c_0(1-\gamma)\epsilon
\]
for a sufficiently small numerical constant $c_0$, then LSVI-UCB returns an
$\epsilon$-optimal policy after
\[
\widetilde O\!\left(
\mathrm{poly}\bigl(d,W,B_\psi,(1-\gamma)^{-1},\epsilon^{-1}\bigr)
\right)
\]
episodes. Under the standard bounded-norm normalization used in linear-MDP
analyses, this specializes to the familiar certificate
\[
\widetilde O\!\left(
\frac{d^2}{(1-\gamma)^4\epsilon^2}
\right).
\]
\end{proposition}

\begin{proof}
Under approximate Bellman linearity, the Bellman target used by LSVI-UCB is
linear in $\widehat\psi$ up to a uniform misspecification error $\eta$. Standard
misspecified-linear-MDP analyses for LSVI-UCB then yield a regret decomposition
of the form
\begin{equation}\label{eq:misspecified-regret}
\mathrm{Regret}(T)
\le
\widetilde O\!\left(
\frac{d\sqrt{T}}{(1-\gamma)^{3/2}}
\right)
+
C_{\mathrm{mis}}\frac{T\eta}{1-\gamma},
\end{equation}
where the hidden factors depend polynomially on the feature and weight norm
bounds. The first term is the usual exact-linear-MDP regret term, and the second
term is the cumulative Bellman misspecification error.

To make the average regret at most $\epsilon$, it suffices to choose $T$ so that
\[
\frac{d}{(1-\gamma)^{3/2}\sqrt T}
\lesssim
\epsilon,
\]
which gives
\[
T
=
\widetilde O\!\left(
\frac{d^2}{(1-\gamma)^3\epsilon^2}
\right).
\]
At this value of $T$, the misspecification term in
\eqref{eq:misspecified-regret} contributes at most order $\epsilon T$ whenever
\[
\eta
\lesssim
(1-\gamma)\epsilon.
\]
Thus the condition $\eta\le c_0(1-\gamma)\epsilon$ ensures that the
misspecification term does not dominate the exact-feature regret term. Converting
the average-regret guarantee to a single output policy by the standard PAC
conversion for discounted problems introduces one additional factor of
$(1-\gamma)^{-1}$, yielding
\[
\widetilde O\!\left(
\frac{d^2}{(1-\gamma)^4\epsilon^2}
\right)
\]
under the usual bounded-norm normalization. Keeping the norm constants explicit
gives the stated polynomial dependence on $d,W,B_\psi,(1-\gamma)^{-1}$ and
$\epsilon^{-1}$.
\end{proof}

\paragraph{Implication for the factored representation.}
Under exact CP realizability, the true feature
\[
\psi^*(s,a)=\phi_s^*(s)\odot\phi_a^*(a)
\]
satisfies Bellman linearity exactly: for every bounded $V$,
\[
r(s,a)+\gamma\E[V(s')\mid s,a]
=
\psi^*(s,a)^\top w_V
\]
for an appropriate coefficient vector $w_V$. Thus the CP feature is a valid
linear-MDP feature. The learned feature $\widehat\psi$ can be used by LSVI-UCB
once it satisfies the stronger uniform condition
\eqref{eq:approx-bellman-linearity}. The representation-learning results in
Appendix~\ref{app:proof-finite-sample} do not by themselves prove this uniform
condition from average $L^2(\nu)$ error; rather, they provide a route for
reducing the number of representation-learning samples needed to obtain an
accurate transition feature.

The potential advantage of the factored structure is therefore conditional but
clear. If the factored and joint learned features are both evaluated through the
same downstream LSVI-UCB analysis, then the online exploration guarantee depends
on the feature dimension and the Bellman misspecification level $\eta$. The
factored representation can improve the upstream sample requirement for reaching
a small $\eta$ because its representation class has a smaller certified
estimation term in the low-CP-bias regime. Once the threshold
$\eta\le c_0(1-\gamma)\epsilon$ is reached, the downstream LSVI-UCB guarantee is
the standard one.


\subsection{Operational Reading of the Covering-Number Bound}
\label{app:operational}

In sample-size terms, the covering number is the number of transitions needed for the contrastive loss to separate two encoder triples within a target tolerance. A smaller covering number means the empirical loss is close to its population value at smaller $n$, so the resulting $\psi=\phi_s\odot\phi_a$ is available as a TD feature at smaller $n$. Take humanoid as a concrete case. The state has $d_s=67$ proprioceptive coordinates, the action has $d_a=21$ actuator commands, and the standard approach fits the transition density on the joint $88$-dimensional input. The CP form fits two smaller pieces instead: a feature of the $67$-dimensional state, and a feature of the $21$-dimensional action that says how the action scales the latent state modes. Theorem~\ref{thm:complexity} states this same split at the level of the covering number.

\section{Ablation Study and Analysis}
\label{app:ablation-analysis}

This appendix gathers two analyses of the FaStR design. Section~\ref{app:nce-ablation} is a controlled encoder ablation that isolates which architectural ingredient drives the empirical gain by interpolating between CTRL-SR and FaStR through two intermediate architectures. Section~\ref{app:width-ablation} varies the encoder hidden width of all three spectral methods. The complementary question, \emph{when the underlying CP assumption is a good fit for the environment}, is treated in Appendix~\ref{app:help-hurt}.

\subsection{Interaction Structure of the Spectral Encoder}
\label{app:nce-ablation}

FaStR differs from CTRL-SR in three design choices: the encoder is split into separate state and action networks, their outputs are combined through element-wise multiplication, and the NCE logit is a diagonal trilinear transition-ratio score rather than a fused state-action logit. We isolate the contribution of each choice with a four-condition ablation in which two intermediate architectures (concat-of-separate and Tucker) sit between the monolithic baseline and FaStR.

\subsubsection{Encoder Structure of the Four Conditions}
\label{app:nce-ablation-impl}

The four conditions share the TD3 backbone, optimization, RP-NCE protocol, and evaluation setup of Appendix~\ref{app:impl} (Sections~\ref{app:impl-protocol} and~\ref{app:impl-fastr}); they differ \emph{only} in the encoder used to produce the spectral feature and, where the encoder structure forces it, in the matching critic.

\paragraph{Condition (a): CTRL-SR.} The baseline of~\citet{gao2025spectral}, identical to Appendix~\ref{app:impl-ctrlsr}: a monolithic encoder $\varphi(s,a)\in\R^{512}$ on $[s;a]$, bilinear NCE logit $\varphi(s,a)^\top\mu(s')$, and an RFF-MLP critic on $\varphi$. Effective NCE dimension $d{=}512$.

\paragraph{Condition (b): Concat-of-separate.} Two independent networks $\phis{:}\,\Scal\to\R^{256}$ and $\phia{:}\,\Acal\to\R^{256}$, concatenated and projected through a learned linear layer with Mish activation: $\psi=\sigma(W[\phis;\,\phia]+b)$, $W\in\R^{512\times512}$. The 256-dim outputs are sized so that the joint feature matches the 512-dim $\psi$ of the other conditions. Each component of $\psi$ depends additively on $\phis$ and $\phia$. The critic is the same RFF-MLP on $\psi$ as in FaStR. Effective NCE dimension $d{=}512$.

\paragraph{Condition (c): Tucker.} Two independent networks $\phis{:}\,\Scal\to\R^{64}$ and $\phia{:}\,\Acal\to\R^{64}$ feeding a full trilinear score: $h(s,a,s')=\phis(s)^\top M(s')\,\phia(a)$ with $M(s')\in\R^{64\times64}$. This is the Tucker parameterization of Section~\ref{sec:method-cp}, with effective NCE dimension $d_s\cdot d_a=4096$. Because full trilinear scoring is incompatible with the Hadamard critic, we use a structurally matched critic: a bilinear term $\phis^\top W\phia$ plus an RFF-MLP residual on $[\phis;\phia]$. The dimension $d_s{=}d_a{=}64$ keeps total parameters and per-step floating-point operations (FLOPs) within 10\% of FaStR's.

\paragraph{Condition (d): FaStR (ours).} Identical to Appendix~\ref{app:impl-fastr}: diagonal trilinear NCE score $(\phis(s)\odot\phia(a))^\top m(s')$ with $\phis,\phia,m\to\R^{512}$, factored critic feature $\psi=\phis(s)\odot\phia(a)$, and an RFF-MLP critic on $\psi$. Effective NCE dimension $d{=}512{=}B$.

The critic architecture varies only where the interaction structure forces it: FaStR's diagonal trilinear score induces the same Hadamard feature used by the critic, while Tucker's full interaction matrix requires a bilinear-residual head.

\subsubsection{Results}
\label{app:nce-ablation-results}

Figure~\ref{fig:nce-ablation} reports learning curves on five DM Control Suite tasks at two action dimensions. Humanoid tasks ($\dim\Acal{=}21$) are the primary comparison because the factored advantage is largest in high-dimensional action spaces; quadruped tasks ($\dim\Acal{=}12$) act as a control where all methods are expected to perform comparably.

\begin{figure}[h]
\centering
\includegraphics[width=\textwidth]{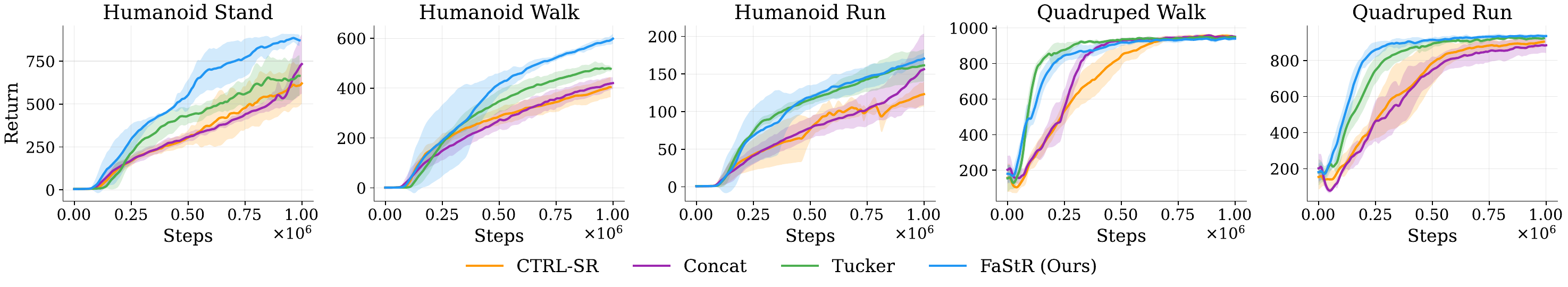}
\caption{NCE transition-ratio score structure ablation. Learning curves on five DM Control Suite tasks at two action dimensions: Humanoid ($\dim\Acal{=}21$) is the primary comparison; Quadruped ($\dim\Acal{=}12$) is a control. Curves: mean return across 8 seeds; shaded: $\pm 1$ std.}
\label{fig:nce-ablation}
\end{figure}

\paragraph{Encoder separation alone does not explain the gain.}
The (a)-vs-(b) contrast isolates the effect of using separate $\phis$ and $\phia$ networks without multiplicative interaction. Concat preserves the split but combines through a learned affine transformation, so each component of $\psi$ is a linear mixture of $\phis$ and $\phia$. On humanoid tasks, concat gives a modest improvement over the monolithic encoder; on quadruped, the two are within noise. The improvement is well below FaStR's gain on the same tasks, ruling out the hypothesis that the split architecture alone accounts for FaStR's advantage.

\paragraph{The diagonal trilinear interaction is the source of the gain.}
The (b)-vs-(d) contrast isolates the CP interaction in the transition-ratio score versus linear projection with encoder separation held constant. FaStR gives higher returns than concat on every humanoid task, despite using the same 512-dimensional $\psi$ and the same RFF-MLP critic. The structural reason is that concat gives an additive scoring function $\psi_k=\sum_j(W_s)_{kj}\phi_{s,j}+\sum_j(W_a)_{kj}\phi_{a,j}+b_k$, mixing all dimensions, while FaStR's NCE logit contains the per-component product $\phi_{s,k}(s)\cdot\phi_{a,k}(a)\cdot m_k(s')$: the $k$-th action feature directly scales the $k$-th state feature before comparison with the $k$-th next-state factor. This per-component product is the algebraic structure required by Theorem~\ref{thm:complexity} for the log covering number to decompose additively over $\Fcal_s$, $\Fcal_a$, $\Fcal_m$. Concat loses this decomposition at the projection layer, where $W$ introduces cross-mode coupling.

\paragraph{Tucker is more expressive but trains worse under NCE.}
The (c)-vs-(d) contrast tests whether the strictly more expressive Tucker trilinear score translates into better contrastive performance. Tucker can represent arbitrary cross-component products $\phi_{s,i}\cdot\phi_{a,j}$, $i\neq j$, that diagonal CP excludes. Empirically Tucker gives lower returns than FaStR on humanoid tasks and matches FaStR on quadruped, where the gradient-coverage constraint is weaker.

The cause is the gradient rank deficiency of Section~\ref{sec:method-cp}: $\nabla_{\vect M(s'_j)}\ell$ lies in $\mathrm{span}\{\phia(a_i)\otimes\phis(s_i)\}_{i=1}^B$, a subspace of dimension at most $\min(B,d_sd_a)$. In condition (c), $d_sd_a{=}4096$ while $B{=}512$, so each mini-batch constrains only $512$ of $4096$ parameter directions of $M(s')$. Three requirements then conflict: encoder capacity, NCE validity ($B\geq d_sd_a$), and gradient coverage in a single batch. At standard batch size, Tucker cannot satisfy all three simultaneously.

\paragraph{Low-rank Tucker does not recover the gap.} The natural fix is to lower $d_s,d_a$ until $d_sd_a\leq B$. The boundary case $d_s{=}d_a{=}22$ ($d_sd_a{=}484\leq B$) on humanoid-walk gives near-zero return: encoder dimension becomes the binding constraint. Any low-rank Tucker variant within the validity constraint lies between this point and dense Tucker, so neither end matches CP.

FaStR avoids the trilemma by restricting $M(s')$ to a diagonal $m(s')\in\R^d$, with $d$ parameters per next-state and gradient spanning all of $\R^d$ whenever $B\geq d$ ($d{=}B{=}512$ here). The rank-one chain-rule structure of Section~\ref{sec:method-cp} holds for any trilinear-score loss $\ell=f(\phis^\top M\phia)$, so the same mechanism applies to the energy-based score matching of SPEDER~\citep{ren2022spectral} and the diffusion objective of Diff-SR~\citep{shribak2024diffusion}.

\paragraph{What CP gives up.}
CP excludes cross-component products $\phi_{s,i}\cdot\phi_{a,j}$ with $i\neq j$, which introduces approximation bias when action dimensions are tightly coupled. Whenever the Tucker core is approximately superdiagonal after a change of basis, the deep encoders can absorb the rotation internally and the CP restriction does not reduce expressiveness. CP's statistical advantage can also outweigh the approximation bias under limited training data, which is the typical online RL regime. Tucker is preferable when the training objective is not constrained by contrastive trilinear scoring (e.g., regression or MLE with enough data to span $\R^{d_sd_a}$), as in prior bilinear-MDP work~\citep{du2021bilinear,yang2020reinforcement}, and when the environment has cross-component couplings that no encoder rotation can absorb.

\subsection{Encoder Hidden-Width Ablation}
\label{app:width-ablation}

The main comparison aligns the hidden width of the spectral methods rather than their exact parameter count (Section~\ref{sec:exp-setup}): at a fixed width, the remaining parameter differences arise from the representation architecture itself, whereas matching parameter counts would require a different width per method and introduce another capacity confound. To check that the conclusions do not depend on the default width, we vary the encoder hidden width for FaStR and both spectral baselines (CTRL-SR and Diff-SR) under a fully controlled protocol. All methods share the same codebase, TD3 backbone, training budget ($1{\times}10^6$ frames), optimization settings, evaluation protocol, and representation dimension ($d{=}512$); only the encoder hidden width is varied ($256/512/1024$), without per-method tuning. Table~\ref{tab:width-ablation} reports the final return (mean over the last 10 evaluations) as mean $\pm$ std over 8 seeds, following the protocol of Table~\ref{tab:full-results}.

\begin{table}[h]
\centering
\caption{Encoder hidden-width ablation for the three spectral methods. Final return, mean $\pm$ std over 8 seeds. Bold: best mean return at each width.}
\label{tab:width-ablation}
\footnotesize
\begin{tabular}{@{}l c ccc@{}}
\toprule
Task & Width & FaStR & CTRL-SR & Diff-SR \\
\midrule
\multirow{3}{*}{Humanoid-Stand} & 256  & $\mathbf{848\pm18}$ & $846\pm80$  & $825\pm44$ \\
                                & 512  & $\mathbf{858\pm26}$ & $752\pm94$  & $761\pm129$ \\
                                & 1024 & $\mathbf{870\pm22}$ & $698\pm119$ & $715\pm42$ \\
\midrule
\multirow{3}{*}{Humanoid-Walk}  & 256  & $\mathbf{583\pm50}$ & $408\pm11$  & $439\pm35$ \\
                                & 512  & $\mathbf{591\pm15}$ & $394\pm37$  & $440\pm93$ \\
                                & 1024 & $\mathbf{743\pm85}$ & $439\pm35$  & $645\pm81$ \\
\midrule
\multirow{3}{*}{Humanoid-Run}   & 256  & $\mathbf{200\pm69}$ & $118\pm65$  & $105\pm70$ \\
                                & 512  & $\mathbf{174\pm19}$ & $118\pm13$  & $107\pm35$ \\
                                & 1024 & $\mathbf{168\pm34}$ & $\;\,98\pm16$ & $118\pm69$ \\
\midrule
\multirow{3}{*}{Dog-Stand}      & 256  & $\mathbf{921\pm15}$ & $855\pm21$  & $819\pm101$ \\
                                & 512  & $\mathbf{955\pm20}$ & $870\pm70$  & $909\pm27$ \\
                                & 1024 & $\mathbf{923\pm44}$ & $877\pm24$  & $910\pm31$ \\
\midrule
\multirow{3}{*}{Dog-Run}        & 256  & $377\pm29$ & $\mathbf{393\pm44}$ & $275\pm60$ \\
                                & 512  & $\mathbf{538\pm49}$ & $505\pm97$  & $310\pm94$ \\
                                & 1024 & $\mathbf{521\pm40}$ & $\mathbf{521\pm43}$ & $450\pm73$ \\
\bottomrule
\end{tabular}
\end{table}

FaStR obtains the best mean return on the Humanoid tasks at all three widths, so the gains do not depend on the default width of $512$. On Humanoid-Stand, both CTRL-SR and Diff-SR become worse as the width increases, and on Humanoid-Walk and Dog-Stand FaStR remains the best at all three widths. Since the ranking is stable across the three widths, it is not specific to the default hidden width. Dog-Run is the only exception: CTRL-SR is slightly ahead at width $256$, and the two methods tie at width $1024$. This is consistent with Section~\ref{sec:exp-singletask}, where stronger state-action coupling in the Dog tasks partly offsets the factored advantage.

\section{When Does the Factorization Help, and When Does It Hurt?}
\label{app:help-hurt}

FaStR is a deliberate structural bias, and it does not provide a uniform improvement over joint encoders. This appendix locates the boundary of its applicability. Section~\ref{app:cp-misfit} is an offline diagnostic that estimates the cost of imposing the CP factorization on the dynamics of each task, independently of the RL training pipeline. Section~\ref{app:complementary-tasks} reports seven additional official DMControl tasks chosen to stress the assumption, with a structural explanation for each outcome. Section~\ref{app:task-coverage} lists every evaluated task.

\subsection{Empirical CP Misfit Diagnostic}
\label{app:cp-misfit}

The finite-sample certificate of Proposition~\ref{prop:finite-sample} decomposes representation error into a CP approximation bias $\epsilon_{\mathrm{CP}}^2$ and an estimation term whose rate depends on the covering number. Proposition~\ref{prop:sample-sep} shows that, under low CP bias, the estimation advantage of the factored class over the joint class scales with $\min(d_s,d_a)$. The single-task experiments are consistent with this prediction but do not measure $\epsilon_{\mathrm{CP}}$ directly. We construct an offline diagnostic that estimates the transition-modeling cost of enforcing CP, independently of the RL training pipeline, providing a structural complement to the per-task return numbers in Table~\ref{tab:full-results}.

\subsubsection{Implementation Details}
\label{app:cp-misfit-impl}

For each of the 13 DM Control Suite tasks in Figure~\ref{fig:cp-misfit}, we collect a replay buffer of 500k transitions from a trained CTRL-SR agent across 5 seeds. Buffers are sourced from CTRL-SR rather than FaStR so that the state-action distribution is not shaped by the CP-factored encoder under evaluation, and they cover a range from early random exploration to trained locomotion, reflecting the RL-relevant state distribution.

We train two transition-modeling predictors on each buffer: the CP predictor scores transitions as $g_{\mathrm{CP}}(s,a,s')=(\phi_s(s)\odot\phi_a(a))^\top m(s')$ with separate state and action encoders, and the joint predictor uses $g_{\mathrm{joint}}(s,a,s')=\varphi(s,a)^\top\mu(s')$ with a single monolithic encoder. Both share feature dimension ($d{=}512$), encoder depth (2 residual blocks with LayerNorm and Mish), learning rate ($10^{-4}$), batch size (512), training epochs (50), and noise schedule (25 VP levels). The CP predictor has more parameters in every task, making the comparison conservative: if it still incurs higher test loss, the gap reflects structural misfit rather than underfitting. Each predictor is trained with the RP-NCE ranking loss plus an auxiliary conditional-moment head (weight 0.1) that predicts standardized next-state features (raw observations and random Fourier projections).

Data is split 70/15/15 into train, validation, and test. We select the checkpoint with the lowest validation NCE loss and report two test-set metrics: the NCE misfit $\widehat\Delta_{\mathrm{CP}}^{\mathrm{NCE}}=\mathcal{L}_{\mathrm{CP}}^{\mathrm{test}}-\mathcal{L}_{\mathrm{joint}}^{\mathrm{test}}$, which matches the contrastive objective used in the RL pipeline, and the moment misfit $\widehat\Delta_{\mathrm{CP}}^{\mathrm{mom}}$, which is independent of NCE. Both are distribution-weighted empirical proxies for $\epsilon_{\mathrm{CP}}^2$ rather than direct estimates.

\subsubsection{Results}
\label{app:cp-misfit-results}

Figure~\ref{fig:cp-misfit} reports the diagnostic across 13 tasks with action dimensions from 2 to 38. The CP factorization incurs a small excess transition-modeling loss in nearly all tasks. Walker-Walk is an outlier on both metrics: it has the highest NCE misfit and the highest moment misfit of any task in the figure.

\begin{figure}[h]
\centering
\includegraphics[width=\textwidth]{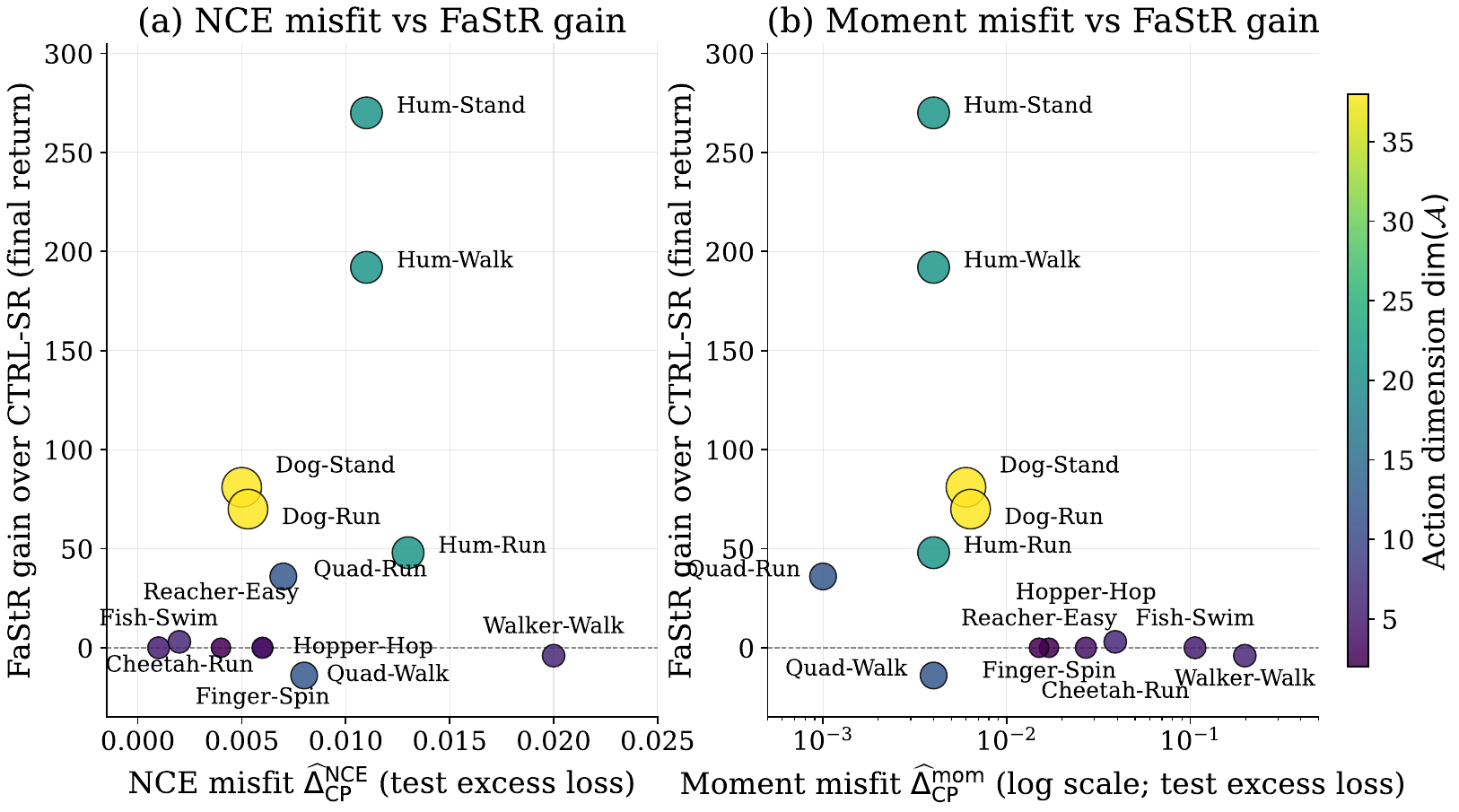}
\caption{Empirical CP misfit diagnostic across 13 DM Control Suite tasks. Each marker is one task; the $x$-axis is the test-set excess loss of the CP predictor over the joint predictor (positive means CP is worse), measured by (a) NCE misfit $\widehat\Delta_{\mathrm{CP}}^{\mathrm{NCE}}$ and (b) moment misfit $\widehat\Delta_{\mathrm{CP}}^{\mathrm{mom}}$ (log scale); the $y$-axis is the final-return gain FaStR $-$ CTRL-SR from Figure~\ref{fig:single-task}. Marker color and size encode action dimension $\dim\Acal$. The CP predictor has more parameters than the joint predictor in every task.}
\label{fig:cp-misfit}
\end{figure}

\paragraph{Two-condition pattern.} The gains match the prediction of Propositions~\ref{prop:finite-sample} and~\ref{prop:sample-sep}. Among low-dimensional tasks ($\dim\Acal\leq6$, lower-left cluster of Figure~\ref{fig:cp-misfit}), CP fits the dynamics well, but the covering-number separation at small $\min(d_s,d_a)$ is too small to yield a measurable return difference. Among high-dimensional tasks ($\dim\Acal\geq12$), where the theory predicts a larger estimation advantage, gains appear in every case where CP misfit is low: the Humanoid and Dog points sit in the upper region of both panels, while the high-misfit Walker-Walk point, with the same action dimension as Cheetah-Run, has no gain, which corresponds to the regime of Remark~\ref{rmk:nonrealizable} in which the bias term offsets the estimation advantage. Neither low misfit nor high action dimensionality alone is sufficient: both conditions must hold together.

\paragraph{Scope of the diagnostic.} The misfit metrics do not rank-order gain magnitudes within the low-misfit group. This is expected, since the estimation advantage depends both on the action dimension and on the effective complexity of the dynamics within the factored class, neither of which is captured by a scalar proxy. We therefore use the diagnostic as a compatibility check for the CP assumption rather than as a continuous predictor of gain magnitude.

\subsection{Complementary Tasks and Failure Modes}
\label{app:complementary-tasks}

The main evaluation targets domains with $\dim\Acal\ge 6$. Here we report seven additional official DMControl tasks outside that set: Acrobot-Swingup, Hopper-Hop, Hopper-Stand, Fish-Swim, Finger-Turn-Hard, Finger-Spin, and Reacher-Hard. They cover manipulation, underactuated control, and low-dimensional locomotion, and were selected to match, as closely as possible within DMControl, the structures under which the CP assumption is expected to be strained: strongly state-dependent underactuation, repeated contact switching, object contact, and whole-body coupling. In these regimes, Proposition~\ref{prop:sample-sep} predicts either a small estimation advantage (small $\min(d_s,d_a)$) or a CP bias large enough to offset it. Finger-Spin, Fish-Swim, and Hopper-Hop already appear among the 13 tasks of the diagnostic in Section~\ref{app:cp-misfit}. All runs follow the protocol of Table~\ref{tab:full-results} (8 seeds, same interaction budget). CTRL-SR is the controlled comparison; SAC, in the 3-layer configuration of Appendix~\ref{app:impl-sac} whose hidden width matches the spectral methods, and TD7 provide additional points of comparison. Results are reported per task rather than in aggregate.

\begin{table}[h]
\centering
\caption{Final return on seven complementary DMControl tasks, mean $\pm$ std over 8 seeds, same protocol as Table~\ref{tab:full-results}. Bold: best mean return in each column.}
\label{tab:complementary}
\footnotesize
\setlength{\tabcolsep}{3pt}
\resizebox{\textwidth}{!}{%
\begin{tabular}{@{}l ccccccc@{}}
\toprule
& {\scriptsize Acrobot Swingup} & {\scriptsize Hopper Hop} & {\scriptsize Hopper Stand} & {\scriptsize Fish Swim} & {\scriptsize Finger Turn Hard} & {\scriptsize Finger Spin} & {\scriptsize Reacher Hard} \\
$\dim\Acal$ & 1 & 4 & 4 & 5 & 2 & 2 & 2 \\
\midrule
FaStR (Ours) & $184\pm70$ & $318\pm124$ & $910\pm41$  & $664\pm60$      & $964\pm19$  & $\mathbf{983\pm29}$  & $970\pm17$ \\
CTRL-SR      & $265\pm63$ & $\mathbf{564\pm36}$  & $919\pm52$  & $660\pm48$      & $\mathbf{972\pm45}$  & $982\pm16$  & $\mathbf{989\pm30}$ \\
SAC(3-layer) & $\mathbf{288\pm87}$ & $209\pm56$  & $\mathbf{943\pm10}$  & $\mathbf{799\pm35}$      & $906\pm59$  & $951\pm35$  & $916\pm44$ \\
TD7          & $\;\,88\pm75$ & $207\pm104$ & $670\pm285$ & $\;\,86\pm166$ & $470\pm399$ & $435\pm178$ & $898\pm12$ \\
\bottomrule
\end{tabular}%
}
\end{table}

FaStR provides no gain on any of the seven tasks. It is $246$ points below CTRL-SR on Hopper-Hop and $81$ below on Acrobot-Swingup, and comparable on the other five, where the two spectral methods are within one standard deviation of each other. Both spectral methods trail SAC on Acrobot-Swingup, Hopper-Stand, and Fish-Swim. Each outcome corresponds to a specific way in which one of the two conditions of the theory is not met.

\paragraph{Acrobot-Swingup: state-dependent underactuation.} The single elbow torque is transmitted through the current joint configuration, so the effect of the same command changes substantially with the pose, and a successful swing-up also requires the action to be timed with the phase of the motion. These strongly state-dependent action effects place more burden on a small set of matched latent factors. At the same time, the action dimension is only $1$, so even a good factorization offers almost no reduction in estimation complexity. The theory therefore predicts no gain in this setting.

\paragraph{Hopper-Hop versus Hopper-Stand: contact switching.} The two Hopper tasks share the same robot, action space, and encoder architecture. During Hopper-Hop the system repeatedly switches between flight and ground support, and the role of an action changes sharply across these contact modes. Encoding this dependence through the state features requires contact-specific latent factors and leaves less reusable low-rank structure. Hopper-Stand, in contrast, spends most of its time in a stable support mode and is only $9$ points below CTRL-SR. Since the embodiment and action dimension are unchanged, this contrast is consistent with frequent contact switching being an important source of structural mismatch.

\paragraph{Finger-Turn-Hard and Finger-Spin: object contact.} These tasks create a related challenge through changing finger--object contact. FaStR provides no gain on either task.

\paragraph{Fish-Swim: whole-body coupling.} The force generated by each fin propagates through the entire deformable body and depends on the global body pose. The effect of an action is therefore distributed across many state variables rather than being reused through a small number of stable local motion patterns. This type of whole-body coupling calls for a richer interaction model.

\paragraph{Reacher-Hard: compatible structure, low dimension.} Reacher-Hard is the complementary control case. It is fully actuated and contact-free, and its dynamics are comparatively compatible with a factored representation. Its action dimension is only $2$, so the joint representation is already inexpensive to estimate and the factorization provides little statistical advantage. As expected, FaStR and CTRL-SR perform similarly.

\paragraph{Summary.} The results separate two cases. When the dynamics require many context-specific state--action interactions, the low-rank structural bias introduces enough approximation error to lower the return (Hopper-Hop, Acrobot-Swingup). On the remaining five tasks the action dimension is at most $5$, the factorization has little estimation complexity to reduce, and the two spectral methods perform alike. Together with the gains on Humanoid and Dog-Stand and the competitive result on Dog-Run, these experiments support the two-condition view: FaStR is most useful when the transition can be compressed into a small number of reusable latent modes \emph{and} the relevant state or action dimension is large enough for this reduction to matter.

\subsection{Task Coverage}
\label{app:task-coverage}

Table~\ref{tab:task-coverage} lists every task on which RL returns are reported, its action dimension, and whether it belongs to the main dimension-scaling evaluation (Section~\ref{sec:exp-singletask}) or to the complementary set (Section~\ref{app:complementary-tasks}). In DMControl, action dimensions of at least $6$ appear only in the articulated locomotion domains of the main set; all other official tasks have action dimensions between $1$ and $5$.

\begin{table}[h]
\centering
\caption{Coverage of evaluated DMControl tasks.}
\label{tab:task-coverage}
\footnotesize
\begin{tabular}{@{}llcl@{}}
\toprule
Domain & Task & $\dim\Acal$ & Evaluation set \\
\midrule
Cheetah   & Run                 & 6  & Main (dimension scaling) \\
Walker    & Walk, Run           & 6  & Main (dimension scaling) \\
Quadruped & Walk, Run           & 12 & Main (dimension scaling) \\
Humanoid  & Stand, Walk, Run    & 21 & Main (dimension scaling) \\
Dog       & Stand, Run          & 38 & Main (dimension scaling) \\
\midrule
Acrobot   & Swingup             & 1  & Complementary \\
Reacher   & Hard                & 2  & Complementary \\
Finger    & Spin, Turn Hard     & 2  & Complementary \\
Hopper    & Stand, Hop          & 4  & Complementary \\
Fish      & Swim                & 5  & Complementary \\
\bottomrule
\end{tabular}
\end{table}

\end{document}